\documentclass[11pt]{article}
\pdfoutput=1

\usepackage[utf8]{inputenc}
\usepackage[T1]{fontenc}
\usepackage[letterpaper,margin=1in]{geometry}
\usepackage{microtype}
\usepackage{xcolor}
\usepackage{graphicx}
\usepackage{subfigure}
\usepackage{booktabs}
\usepackage{amsmath,amssymb,mathtools,amsthm}
\usepackage[round]{natbib}
\usepackage{url}
\usepackage{hyperref}
\usepackage[capitalize,noabbrev]{cleveref}
\hypersetup{
  pdftitle={Active Inference as Context Acquisition for AI Agents},
  pdfauthor={Sanchayan Dutta, Sai Niranjan Ramachandran, Suvrit Sra},
  pdfsubject={Active inference for context acquisition in AI agents},
  pdfkeywords={active inference, context acquisition, AI agents, expected free energy, optimal question asking}
}
\usepackage{placeins}
\usepackage{listings}
\usepackage[most]{tcolorbox}
\tcbuselibrary{listings,breakable}
\usepackage{ss}

\lstdefinestyle{oqalist}{
  basicstyle=\ttfamily\footnotesize,
  columns=fullflexible,
  breaklines=true,
  breakatwhitespace=false,
  breakindent=0pt,
  breakautoindent=false,
  postbreak=\mbox{},
  showstringspaces=false,
  keepspaces=true
}

\newtcblisting{oqacode}[3][]{%
  enhanced,
  breakable,
  colback=#2!4,
  colframe=#2!60!black,
  boxrule=0.6pt,
  arc=2pt,
  left=6pt,right=6pt,top=4pt,bottom=4pt,
  title={#3},
  fonttitle=\bfseries,
  listing only,
  listing options={style=oqalist, breakautoindent=false, breakindent=0pt},
  #1
}

\newtcolorbox{oqafig}[3][]{%
  enhanced,
  breakable,
  colback=#2!4,
  colframe=#2!60!black,
  boxrule=0.6pt,
  arc=2pt,
  left=6pt,right=6pt,top=4pt,bottom=4pt,
  title={#3},
  fonttitle=\bfseries,
  #1
}

\theoremstyle{plain}
\newtheorem{theorem}{Theorem}[section]
\newtheorem{proposition}[theorem]{Proposition}

\theoremstyle{definition}

\theoremstyle{remark}

\title{Active Inference as Context Acquisition for AI Agents}
\author{%
\name Sanchayan Dutta \email{dutta@ucdavis.edu}\\
\addr Department of Mathematics, University of California, Davis\\[2pt]
\name Sai Niranjan Ramachandran \email{sainiranjan.ramachandran@tum.de}\\
\addr Technical University of Munich, Garching, Germany\\[2pt]
\name Suvrit Sra \email{s.sra@tum.de}\\
\addr Technical University of Munich, Garching, Germany
}

\begin{document}
\maketitle

\begin{abstract}
Interactive AI agents must acquire the right context as efficiently as possible. When a user omits a constraint, preference, file, or task variable, an agent can proceed with a default assumption or spend tokens on a clarifying question, retrieval call, tool call, or prompt trial. We formulate this tradeoff as active inference for context acquisition. An inner inference step updates beliefs over a latent task state, and an outer decision selects the next context action, task action, or stop action to minimize expected free energy under cost. In deterministic settings, the epistemic term reduces to expected information gain, optionally normalized by token cost. We instantiate the framework in Optimal Question Asking (OQA), with exact posteriors and a dynamic programming oracle, and benchmark frontier language models on binary and multiway categorical tasks from 25 to 300 candidates. We also study clarification before generation and automated prompt optimization under token budgets. The formulation is model-agnostic and views active inference as a design principle for the context-acquisition layer of AI agents.
\end{abstract}

\section{Introduction}
Users of AI systems rarely specify everything up front. A constraint may be omitted or a preference left implicit. Faced with such inputs, the system can commit early and risk being wrong, or it can acquire more context first. The context action might be a clarifying question, retrieval call, tool call, prompt evaluation, or visual inspection. Each such action has a price. It costs tokens, adds latency, and can require effort from a user or environment.

We study this tradeoff using active inference. Let $x$ be the latent task state, such as intent, target, preference, missing constraint, or best prompt. Let $a$ be either a context action or a task action, and let $o$ be the resulting reply, trace, or observation, modeled by $p(o \mid x,a)$. For each candidate action, an inner inference step updates beliefs over $x$ under possible observations. An outer step then chooses the action that minimizes expected free energy, including risk, epistemic value, and cost. Under exact inference and deterministic observations, the epistemic term is expected information gain. The rule is simple. Acquire more context only when the expected reduction in relevant uncertainty justifies its cost.

To make uncertainty measurable, we instantiate the idea in \emph{Optimal Question Asking} (OQA). OQA turns question asking into a controlled Twenty Questions game \citep{jedynak2012twenty} over a finite attribute table. A hidden target is sampled uniformly. The agent asks about one listed attribute per turn. Answers are obtained by table lookup, either yes/no for binary tables or categorical values for $k$-ary tables. Each transcript induces a remaining consistent set $C_t$, or an equivalence class when several items share the same attribute vector. With noiseless answers, the posterior is uniform over $C_t$, so uncertainty is exact and entropy is $\log_2 |C_t|$. We also compute the optimal expected number of questions under the same allowed queries using dynamic programming. This gives a direct planning gap against an oracle.

We then test the same accounting view in two token-budgeted prompting workflows. In \emph{prompt autocompletion}, the model may spend a few clarification turns before writing a final description. In \emph{automated prompt optimization}, the model allocates a fixed evaluation budget across prompt variants and uses outcomes to infer which prompt is best. In both cases, tokens buy evidence about a hidden choice that matters for the final decision.

This framing is close to reinforcement learning and control-as-inference \citep{levine2018controlasinference,millidge2020relationship}. We do not treat active inference as a rival to RL. The distinction is operational. In context-acquisition settings, the agent can choose the next measurement and thereby shape the evidence stream before acting. Expected free energy gives a compact language for that choice.

\subsection{Main Contributions}\label{sec:main-contributions}
\begin{itemize}
\setlength{\itemsep}{0pt}
\setlength{\parskip}{0pt}
\setlength{\parsep}{0pt}
\item \textbf{Context acquisition as active inference.}
We cast questions, retrievals, tool calls, prompt trials, and inspections as actions whose purpose is to improve beliefs before a task action.

\item \textbf{A bilevel one-step formulation.}
We separate the inner belief update under a hypothetical observation from the outer choice of the next context action, task action, or stop action that minimizes expected free energy under cost.

\item \textbf{OQA, a benchmark with exact uncertainty.}
We introduce \emph{Optimal Question Asking} (OQA), a deterministic attribute-table benchmark in which a model identifies a hidden target by asking informative questions.

\item \textbf{Oracle baselines and stopping rules.}
For a fixed table, we compute an optimal questioning strategy using dynamic programming, and state value-of-information conditions for when context acquisition is worth its cost.

\item \textbf{Two prompting case studies with token budgets.}
We apply the same accounting view to clarification before generation and automated prompt optimization, reading token spending as active experiment selection.
\end{itemize}

We validate these ideas empirically in OQA on seven frontier models\footnote{Models: GPT-5, GPT-4.1, Gemini 2.5 Pro, Gemini 2.0 Flash, Claude Sonnet 4.5, Claude Haiku 4.5, and Grok 4.} via API, and in both prompting studies.

In the appendix, we include a discussion of how the bilevel view of one-step active inference can be used to mathematically model adaptive prompt attacks and possible defenses. We also discuss the structural similarities and differences between active inference and bilevel reinforcement learning.

\section{Related Work}\label{sec:related}
Active inference \citep{friston2017process,parr2019generalised} extends classical inference by incorporating ideas from control theory. The core idea is to select actions that maximize information gain through explicit decomposition of preference and epistemic uncertainty over hidden states. This information-theoretic perspective connects to a rich literature on optimal querying and exploration. 

Expected information gain is fundamental in Bayesian experimental design \citep{lindley1956information,chaloner1995design,rainforth2024modern} and drives action selection in active learning and Bayesian optimization through analogous mutual information objectives \citep{houlsby2011bald,hennig2012entropy}. Similarly, information-directed sampling in bandits balances regret against information acquisition \citep{russo2014ids}. A closely related line of work establishes stopping rules for interactions as a function of query cost \citep{haertel2008costs,bloodgood2009stopping}.

\paragraph{LLM information seeking and clarification.}
Recent work has begun to treat LLMs as active information seekers rather than passive respondents. UoT uses uncertainty-aware simulation and information-gain-motivated rewards to select follow-up questions \citep{hu2024uot}. Active preference inference uses probabilistic models induced by LLM prompts to ask informative preference questions \citep{piriyakulkij2023activepref}. CLAMBER evaluates whether models detect and clarify ambiguous information needs \citep{zhang2024clamber}, while Ask-when-Needed studies tool-use failures under unclear instructions \citep{wang2024learningask}. Active Task Disambiguation and BED-LLM both frame clarification as Bayesian experimental design for LLM agents \citep{kobalczyk2025active, choudhury2025bedllm}. Complementary work on termination, such as CaRT, studies when an agent has gathered enough information to stop and act \citep{liu2025cart}. Our work connects these directions to the active-inference expected-free-energy decomposition and adds exact oracle diagnostics through OQA.

Our setting, optimal question asking (OQA), instantiates these ideas in a Twenty Questions framework. Here, successive queries partition a hypothesis set until one target remains \citep{jedynak2012twenty}. Generalized binary search provides optimal splitting strategies under deterministic answers \citep{nowak2011geometry}, with extensions handling noisy responses and group identification \citep{nowak2009noisy,bellala2010extensions}. Algorithmically, our approach mirrors the bilevel optimization structure of meta-learning: an inner loop updates beliefs given an answer, while an outer loop selects which question maximizes expected utility \citep{colson2007bilevel,franceschi2018bilevel}. This parallels intrinsic motivation in RL, where information gain drives exploration \citep{houthooft2016vime,pathak2017curiosity}, and prompt optimization methods that search over discrete or continuous prompt spaces \citep{shin2020autoprompt,zhou2023ape}. Active inference agents have been deployed in similar settings, from OpenAI Gym environments \citep{cullen2018openai_gym} to cognitive models of visual search \citep{cullen_meta_bayesian_visual_search}.

\paragraph{Active inference and reinforcement learning.}
A sharp separation between active inference and RL would be misleading. Control-as-inference shows how RL objectives can be represented as probabilistic inference problems \citep{levine2018controlasinference}, and Millidge and coauthors show that active inference and control-as-inference differ mainly in how rewards, goals, or preferences are encoded in the graphical model \citep{millidge2020relationship,millidge2021thesis}. Millidge's retrospective makes the same point in plain language, while also emphasizing that the expected-free-energy objective foregrounds information gain and exploration \citep{millidge2024retrospective,millidge2021origin}. Our contribution is therefore not the claim that active inference is categorically different from RL. It is the claim that the active-inference decomposition gives a compact design language for LLM systems that can choose what evidence to buy before acting.

\begin{table}[!htbp]
\centering
\footnotesize
\setlength{\tabcolsep}{3.5pt}
\renewcommand{\arraystretch}{0.92}
\caption{Operational distinction used in this paper. The contrast is not categorical: RL can include information bonuses, and active inference can use RL solvers.}
\label{tab:rl-aif-differences}
\begin{tabular}{p{0.18\linewidth}p{0.38\linewidth}p{0.36\linewidth}}
\toprule
Aspect & Generic RL view & Active-inference view used here \\
\midrule
Control & Optimize actions for return under the sampled or observed data stream. & Choose actions that also determine what evidence arrives next. \\
Exploration & Usually added through bonuses, entropy, optimism, or posterior sampling. & Appears directly as epistemic value in expected free energy. \\
Belief and queries & Beliefs may be implicit in a value function, recurrent state, or model posterior. & An explicit posterior over intent, target, or best prompt makes queries first-class actions. \\
Evaluation & Return, regret, accuracy, or reward-model score. & Entropy drop, oracle gap, bits per token, and final task success. \\
\bottomrule
\end{tabular}
\vspace{-0.06in}
\end{table}

\section{One-Step Active Inference as Bilevel Optimization}\label{sec:ai-bilevel}
In single step active inference we separate belief update from action selection. Let $x\in\mathcal{X}$ be a latent state, let $a\in\mathcal{A}$ be an action, and let $o\in\mathcal{O}$ be the next observation. Given a likelihood $p_\theta(o\mid x,a)$ and a current belief $q_t(x)$, define the predictive distribution
\begin{align}
q_t(o\mid a)
&= \int p_\theta(o\mid x,a)\, q_t(x)\, dx,
\label{eq:predictive}\\
q_t(x\mid o,a)
&= \frac{p_\theta(o\mid x,a)\, q_t(x)}{q_t(o\mid a)},
\label{eq:bayes-posterior}
\end{align}
assuming $q_t(o\mid a)>0$ for the relevant $o$. In discrete settings, we can replace integrals with sums. Unless stated otherwise, $\log$ denotes the natural logarithm, so KL and mutual information are in nats; OQA reports entropy in bits. 

Fix $(a,o)$ and let $q(x)\in\mathcal{Q}$ be a candidate posterior. We define the variational free energy
\begin{align}
F_t(q; a,o)
&= \mathbb{E}_{q(x)}\!\Big[\log q(x) - \log p_\theta(o\mid x,a) - \log q_t(x)\Big]
\label{eq:free-energy}\\
&= -\log q_t(o\mid a)
+ \mathrm{KL}\!\big(q(x)\,\|\, q_t(x\mid o,a)\big).
\label{eq:free-energy-decomp}
\end{align}
The inner update is
\begin{equation}
q_{t+1}^\star(\cdot; a,o)\in \arg\min_{q\in\mathcal{Q}} F_t(q; a,o).
\label{eq:inner-argmin}
\end{equation}
If $\mathcal{Q}$ contains $q_t(\cdot\mid o,a)$, then $q_{t+1}^\star(x; a,o)=q_t(x\mid o,a)$. In the proof of concept experiments, $\mathcal{Q}$ contains the exact posterior, so the inner update is closed form.

Preferences over observations are represented by a distribution $p^\star(o)$. A standard one step expected free energy includes a term that penalizes ambiguity in the observation model through $H(o\mid x,a)$. In the deterministic answer settings used in OQA, this ambiguity term is zero, and the action score can be written as
\begin{equation}
\begin{aligned}
G_t(a)
&= \underbrace{\mathbb{E}_{o \sim q_t(\cdot \mid a)}\!\big[-\log p^\star(o)\big]}_{\text{expected risk}}
\\
&\quad - \underbrace{\mathbb{E}_{o \sim q_t(\cdot \mid a)}\!
\Big[\mathrm{KL}\!\big(q_{t+1}^\star(\cdot; a,o)\,\|\, q_t(\cdot)\big)\Big]}_{\text{expected KL change in belief}}.
\end{aligned}
\label{eq:efe-risk-ig}
\end{equation}
With an exact inner update, the second term is an information gain term, as made explicit below.

\begin{proposition}[Mutual-information form and deterministic special case]\label{prop:kl-mi}
Assume the inner solution is exact, $q_{t+1}^\star(x; a,o)=q_t(x\mid o,a)$, with predictive $q_t(o\mid a)$ from \eqref{eq:predictive}. Then
\begin{equation}
\mathbb{E}_{o \sim q_t(\cdot \mid a)}\!
\Big[\mathrm{KL}\!\big(q_{t+1}^\star(\cdot; a,o)\,\|\, q_t(\cdot)\big)\Big]
= I_{q_t}(x; o \mid a),
\label{eq:ig-mi}
\end{equation}
the mutual information under the joint $q_t(x)\,p_\theta(o\mid x,a)$. If $p^\star$ is constant, the risk term in \eqref{eq:efe-risk-ig} is independent of $a$, so
\begin{equation}
a^\star\in\arg\min_a G_t(a)
\quad\Longleftrightarrow\quad
a^\star\in\arg\max_a I_{q_t}(x;o\mid a).
\label{eq:pure-epistemic}
\end{equation}
Moreover,
\begin{equation}
I_{q_t}(x;o\mid a)=H_{q_t}(o\mid a)-\mathbb{E}_{x\sim q_t} H(o\mid x,a).
\label{eq:deterministic-entropy}
\end{equation}
If $H(o\mid x,a)=0$ for $q_t$-a.e.\ $x$, for example in discrete deterministic observation models, then $I_{q_t}(x;o\mid a)=H_{q_t}(o\mid a)$.
\end{proposition}

\begin{proof}
By definition under the joint $q_t(x)\,p_\theta(o\mid x,a)$,
\[
I_{q_t}(x;o\mid a)=\mathbb{E}_{o\sim q_t(\cdot\mid a)}\!\big[\mathrm{KL}(q_t(x\mid o,a)\|q_t(x))\big].
\]
Substituting $q_{t+1}^\star=q_t(\cdot\mid o,a)$ gives \eqref{eq:ig-mi}. If $p^\star$ is constant, $\mathbb{E}_{o}[-\log p^\star(o)]$ does not depend on $a$, which yields \eqref{eq:pure-epistemic}. The identity \eqref{eq:deterministic-entropy} is the standard entropy decomposition of mutual information.
\end{proof}

The dependence of \eqref{eq:efe-risk-ig} on the optimizer in \eqref{eq:inner-argmin} yields the bilevel program:
\begin{equation}
\label{eq:bilevel}
\boxed{%
\begin{aligned}
\min_{a \in \mathcal{A}} \quad & G_t(a)\\
\text{s.t.}\quad
& q_{t+1}^\star(\cdot; a,o)\in
\arg\min_{q\in\mathcal{Q}}\, F_t(q; a,o),
\\
& \forall o\in\mathcal{O}\ \text{with}\ q_t(o\mid a)>0 .
\end{aligned}}
\end{equation}
The outer expectation in $G_t(a)$ is taken over $o \sim q_t(\cdot\mid a)$.

\begin{proposition}[Differentiability of the bilevel objective]\label{prop:bilevel-differentiable}
Let $\mathcal{Q}=\{q_\phi:\phi\in\mathbb{R}^m\}$ and write $F_t(\phi; a,o):=F_t(q_\phi; a,o)$. Define
\begin{equation}
\begin{aligned}
\phi^\star(a,o) &\in \arg\min_{\phi} F_t(\phi; a,o),\\
G_t(a) &= \mathbb{E}_{o\sim q_t(\cdot\mid a)}\big[\mathcal{L}(a,o,\phi^\star(a,o))\big],
\end{aligned}
\end{equation}
for a differentiable outer integrand $\mathcal{L}$. Suppose that for relevant $(a,o)$ the minimizer $\phi^\star(a,o)$ is isolated, that $F_t(\phi; a,o)$ is twice continuously differentiable in $(a,\phi)$, and that $\nabla^2_{\phi\phi}F_t(\phi^\star; a,o)\succ 0$. Then $\phi^\star(a,o)$ is locally differentiable in $a$. If, in addition, $q_t(o\mid a)$ and $\mathcal{L}$ satisfy standard conditions for differentiating under the expectation, then $G_t(a)$ is differentiable. Gradients can be computed by implicit differentiation or by unrolling an inner solver. For discrete $\mathcal{A}$, \eqref{eq:efe-risk-ig} can be evaluated by enumeration, or by a differentiable relaxation.
\end{proposition}

\begin{proof}[Proof sketch]
At an isolated minimizer, $\nabla_\phi F_t(\phi^\star; a,o)=0$. If $\nabla^2_{\phi\phi}F_t(\phi^\star; a,o)\succ 0$, the Jacobian with respect to $\phi$ is invertible, so the implicit function theorem yields local differentiability of $\phi^\star(a,o)$ in $a$ and the standard implicit gradient expression.
\end{proof}

\begin{proposition}[Information-per-cost stopping rule]\label{prop:cost-stopping}
Suppose the inner update is exact, preferences are constant over the possible observations of each information-gathering action, and taking no further measurement has score $0$. Let action $a$ have expected cost $c(a)>0$, and let the one-step objective be expected posterior entropy plus a cost penalty $\lambda c(a)$, with $\lambda>0$. Then an information-gathering action is worthwhile exactly when
\begin{equation}
I_{q_t}(x;o\mid a) > \lambda c(a).
\label{eq:worthwhile-query}
\end{equation}
Among worthwhile actions, the optimal one maximizes $I_{q_t}(x;o\mid a)-\lambda c(a)$. Equivalently, if a fixed cost budget is spent greedily in small units, the myopic selector ranks actions by information gained per unit cost.
\end{proposition}

\begin{proof}
For a fixed current belief, $H_{q_t}(x)$ is independent of the next action. Exact Bayes gives
\[
\mathbb{E}_{o\sim q_t(\cdot\mid a)}H_{q_t}(x\mid o,a)=H_{q_t}(x)-I_{q_t}(x;o\mid a).
\]
Adding the cost penalty yields the one-step score $H_{q_t}(x)-I_{q_t}(x;o\mid a)+\lambda c(a)$. Subtracting the no-measurement score $H_{q_t}(x)$ gives $\lambda c(a)-I_{q_t}(x;o\mid a)$. Thus asking is beneficial precisely when \eqref{eq:worthwhile-query} holds, and the maximizing rule follows by dropping the action-independent entropy term.
\end{proof}

\begin{theorem}[Bayes-risk value of context]\label{thm:bayes-risk-voi}
Let $\mathcal{D}$ be a set of terminal decisions and let $\ell(d,x)$ be a finite loss. For a belief $q$ over $x$, define the Bayes risk
\begin{equation}
R(q)=\min_{d\in\mathcal{D}} \mathbb{E}_{x\sim q}\,\ell(d,x).
\label{eq:bayes-risk}
\end{equation}
Consider a context-acquisition action $a$ with cost $c(a)>0$, predictive distribution $q_t(o\mid a)$, and exact posterior $q_t(\cdot\mid o,a)$ after observing $o$. Suppose the agent may either act immediately, incurring risk $R(q_t)$, or take one context action $a$ and then act optimally under the resulting posterior, incurring expected score
\begin{equation}
\mathbb{E}_{o\sim q_t(\cdot\mid a)} R(q_t(\cdot\mid o,a)) + \lambda c(a).
\end{equation}
Then $a$ is strictly preferred to acting now exactly when
\begin{equation}
R(q_t)-\mathbb{E}_{o\sim q_t(\cdot\mid a)} R(q_t(\cdot\mid o,a)) > \lambda c(a).
\label{eq:voi-threshold}
\end{equation}
If the terminal loss is logarithmic and the terminal decision is a predictive distribution, then $R(q)=H(q)$ and the left side of \eqref{eq:voi-threshold} is $I_{q_t}(x;o\mid a)$.
\end{theorem}

\begin{proof}
Acting immediately gives the optimal terminal risk $R(q_t)$ by definition. Taking $a$ first gives an observation-dependent posterior and then the optimal terminal risk for that posterior, plus the cost penalty. Comparing the two scores yields \eqref{eq:voi-threshold}. For logarithmic loss, the Bayes-optimal predictive distribution is the belief itself and the Bayes risk is entropy. The expected entropy reduction identity gives
\[
H(q_t)-\mathbb{E}_{o\sim q_t(\cdot\mid a)} H(q_t(\cdot\mid o,a))=I_{q_t}(x;o\mid a),
\]
which proves the final claim.
\end{proof}

\section{Active Inference as Context Acquisition}\label{sec:context-acquisition}
The preceding objective can be read as an inference-time control layer for AI agents. The agent maintains a belief over latent task state, then chooses among context actions, task actions, and stopping. A context action is any move that primarily changes the information available to the agent: asking the user, retrieving a passage, calling a tool to fill a missing field, evaluating a prompt variant, or inspecting part of an input. A task action is the move that commits to the requested output.

This view separates three roles that are often mixed in prompting practice. The \emph{belief state} stores what the agent currently thinks is true about the task. The \emph{measurement model} says what each context action could reveal. The \emph{preference or loss model} says which terminal outcomes matter. Expected free energy then scores a candidate context action by how much it is expected to reduce decision-relevant uncertainty, not merely by whether it seems like a reasonable question.

\begin{oqafig}{blue}{Algorithm 1: AIF-Context at inference time}
At turn $t$, maintain a belief $q_t(x)$ over latent task state.
\begin{enumerate}
\setlength{\itemsep}{0pt}
\item Propose a small action set $\mathcal{A}_t^{\mathrm{ctx}}$ of context actions and a set $\mathcal{A}_t^{\mathrm{task}}$ of terminal task actions.
\item For each $a\in\mathcal{A}_t^{\mathrm{ctx}}$, predict possible observations with $q_t(o\mid a)$ and compute the posterior $q_t(x\mid o,a)$ that would follow each observation.
\item Score $a$ by expected free energy, or by the Bayes-risk value of context in \Cref{thm:bayes-risk-voi} when a terminal loss is available.
\item If no context action beats acting now after cost, choose the best task action. Otherwise execute the best context action, observe $o_t$, update $q_{t+1}$, and repeat.
\end{enumerate}
\end{oqafig}

OQA instantiates this loop with a finite candidate set, deterministic measurements, and exact posteriors. Prompt autocompletion instantiates it with a latent style variable and fixed clarification templates. Prompt optimization instantiates it with a latent best-prompt identity and automatically scored trials. These are deliberately restricted cases, but the abstraction is broader: an AI agent can treat every token, query, tool call, or inspection as an experiment that competes against acting now.

The framework also clarifies what OQA does not claim. In deterministic OQA with flat observation preferences, one-step expected free energy reduces to greedy expected information gain. Greedy information gain is a useful local rule, but it is not the same as the globally optimal policy over a full dialogue. The dynamic programming oracle therefore plays a conceptual role beyond benchmarking: it measures the cost of myopic context acquisition.

\section{Binary OQA Experiments: 25 and 100 Candidates}\label{sec:binary-oqa}
\begin{figure}[!htbp]\centering

  \includegraphics[width=0.32\linewidth]{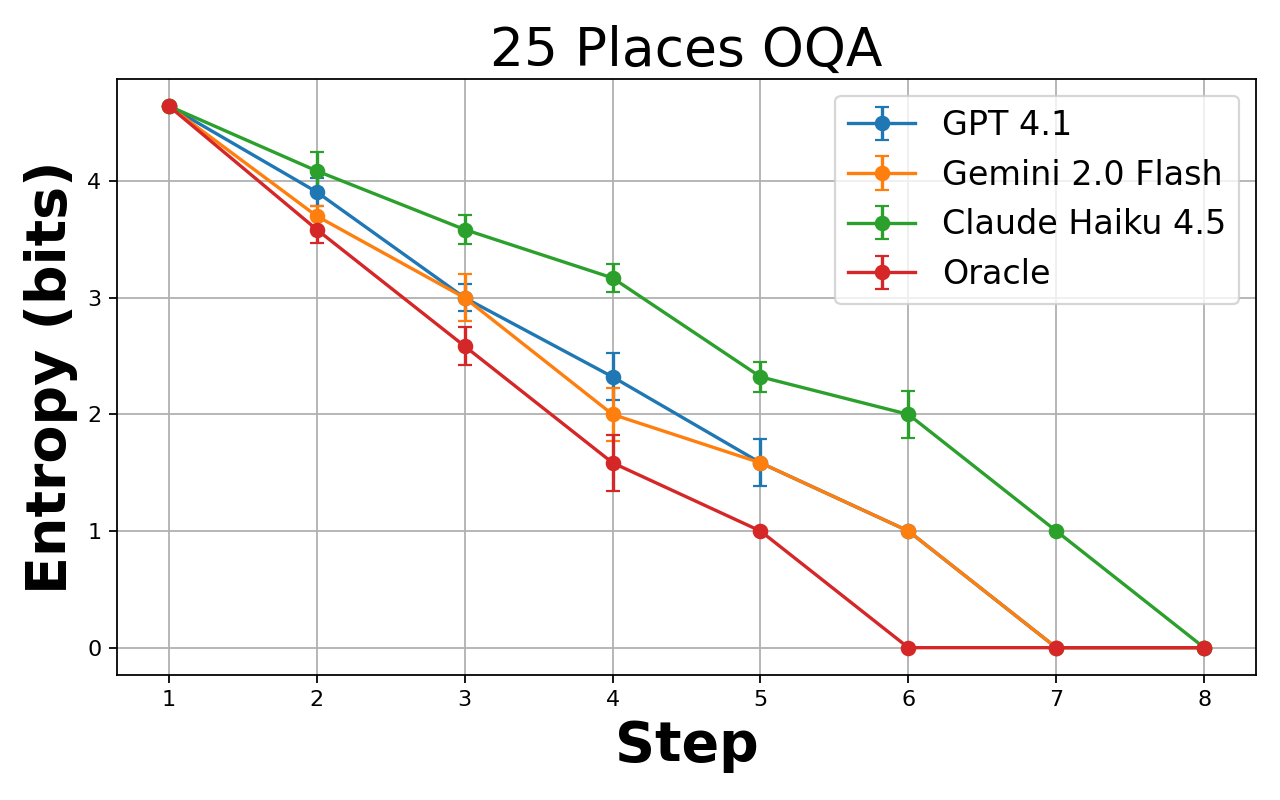}
  \includegraphics[width=0.32\linewidth]{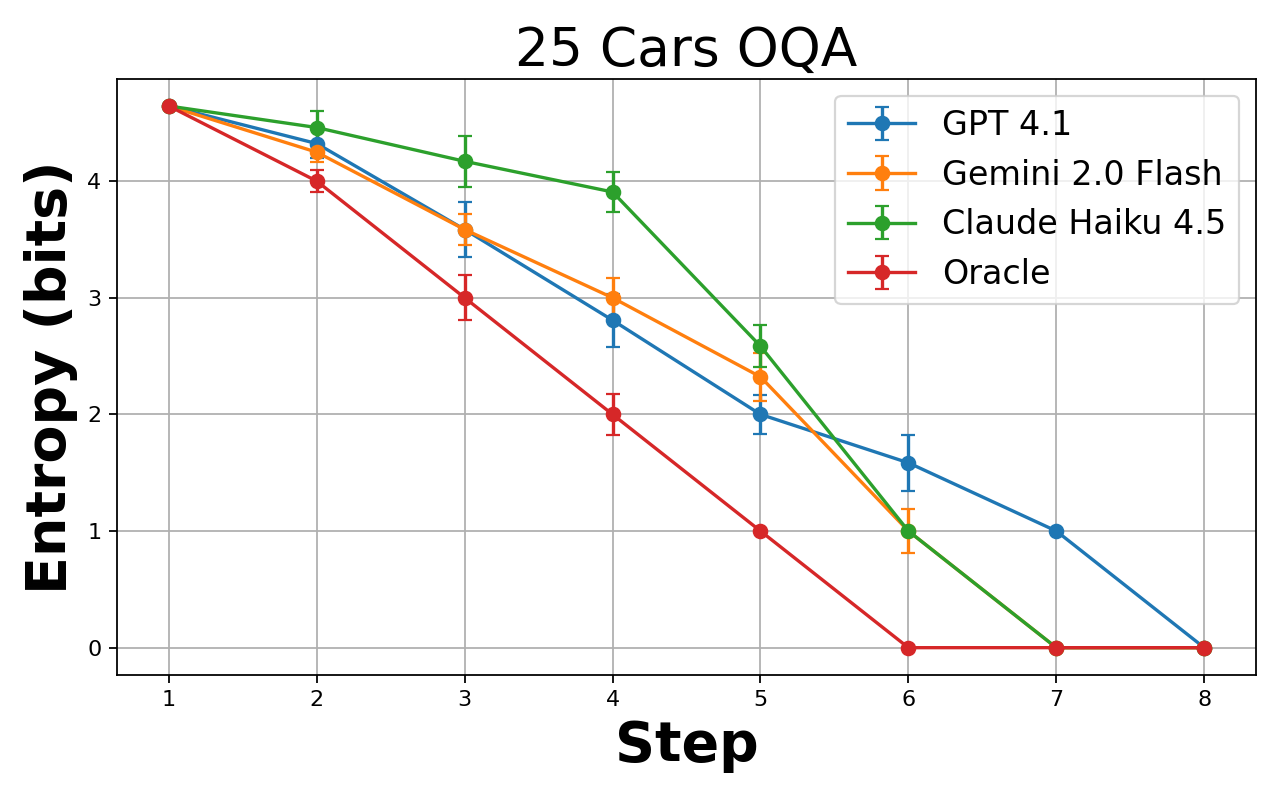}
  \includegraphics[width=0.32\linewidth]{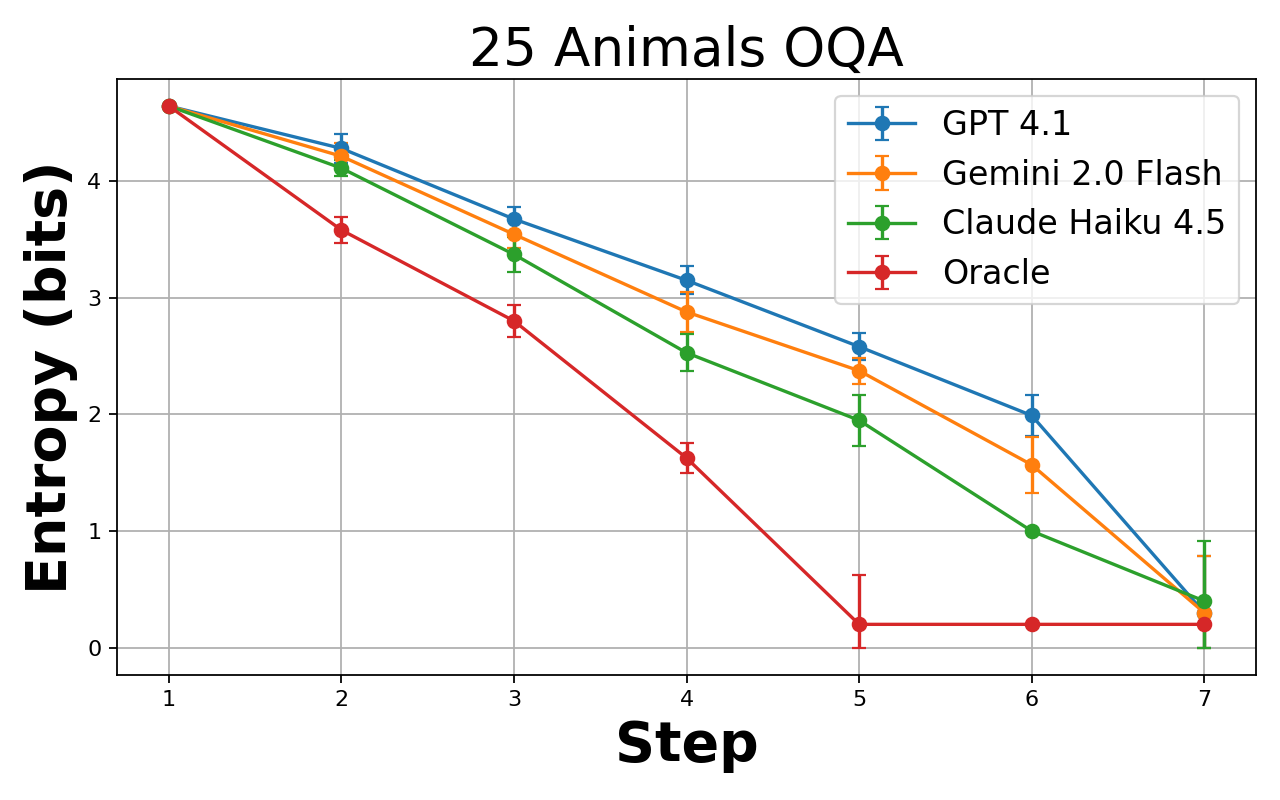}
\caption{Binary OQA with $N=25$. Curves show mean entropy $H_t=\log_2|C_t|$ across uniformly sampled targets, with $\pm 1$ standard deviation, along with the DP oracle.}
\label{fig:binary-oqa-25}
\end{figure}

\begin{figure}[!htbp]\centering
  \includegraphics[width=0.32\linewidth]{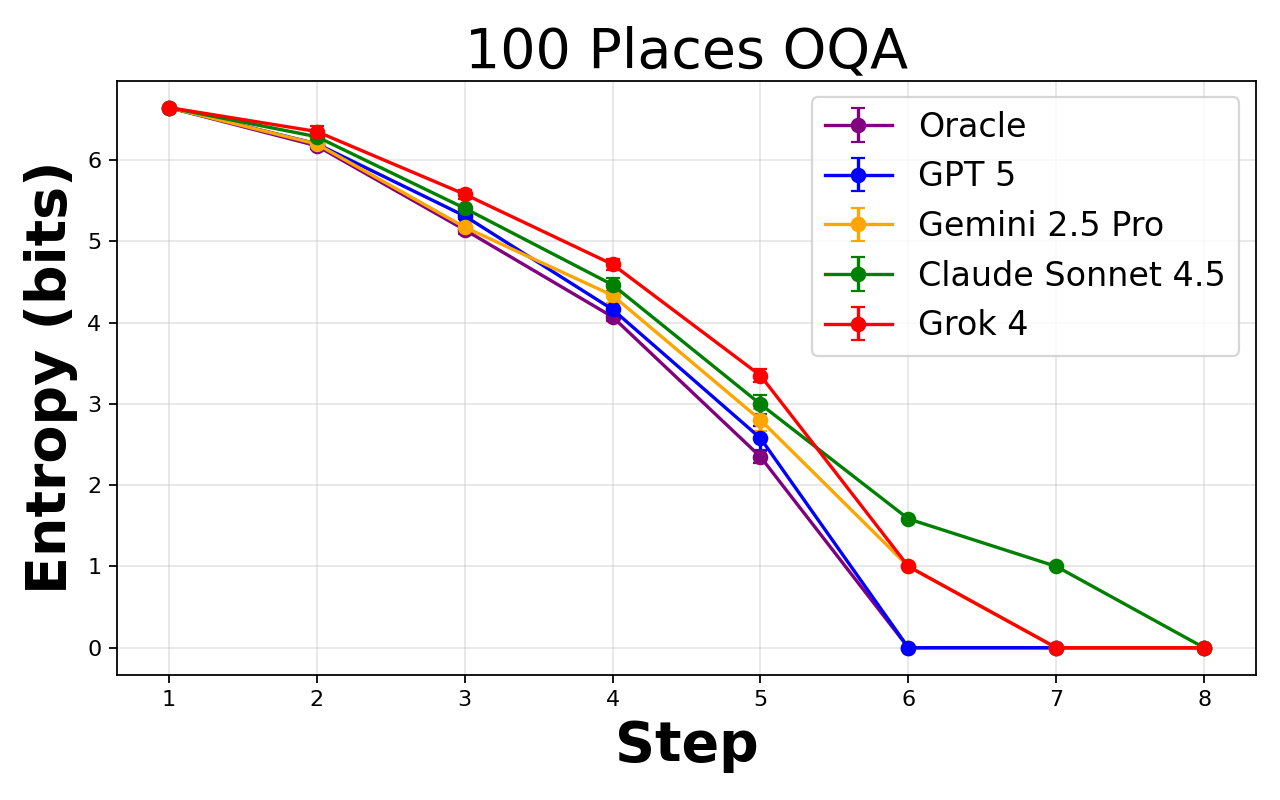}
%
  \includegraphics[width=0.32\linewidth]{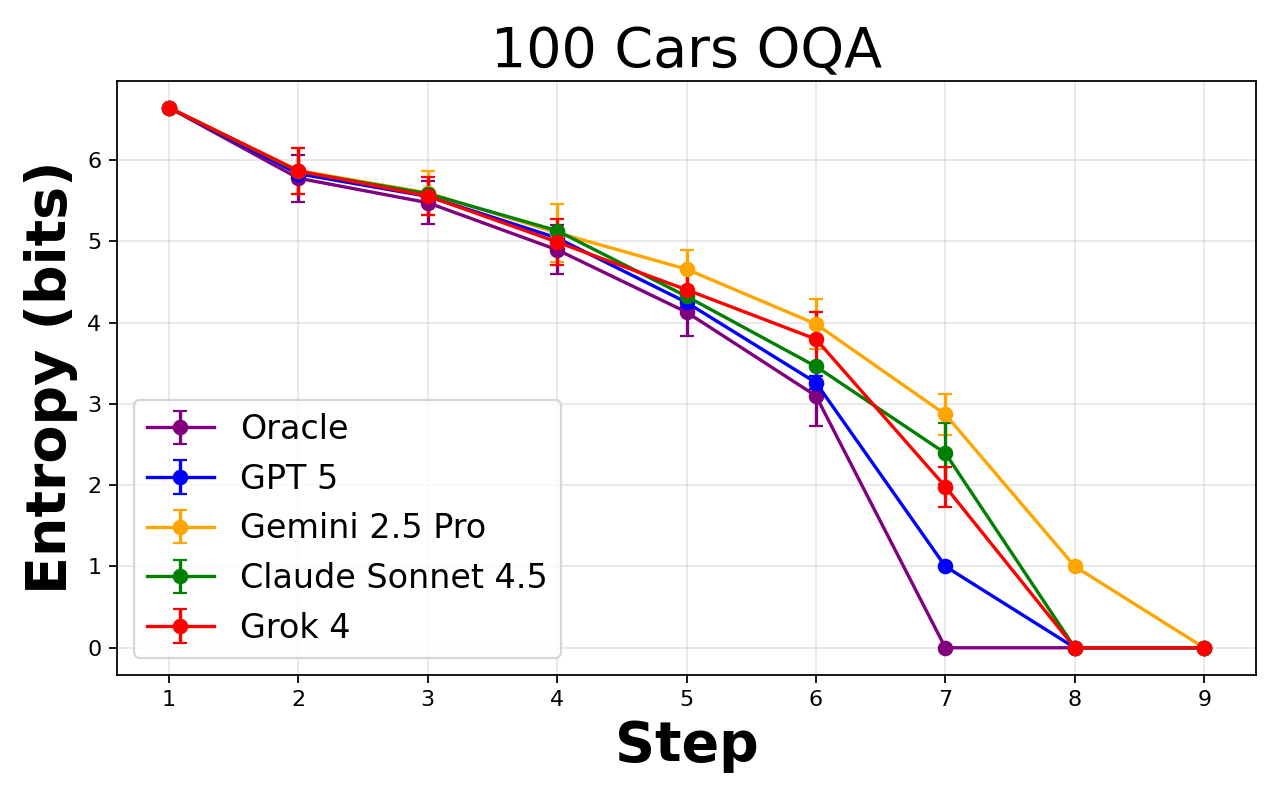}
%
  \includegraphics[width=0.32\linewidth]{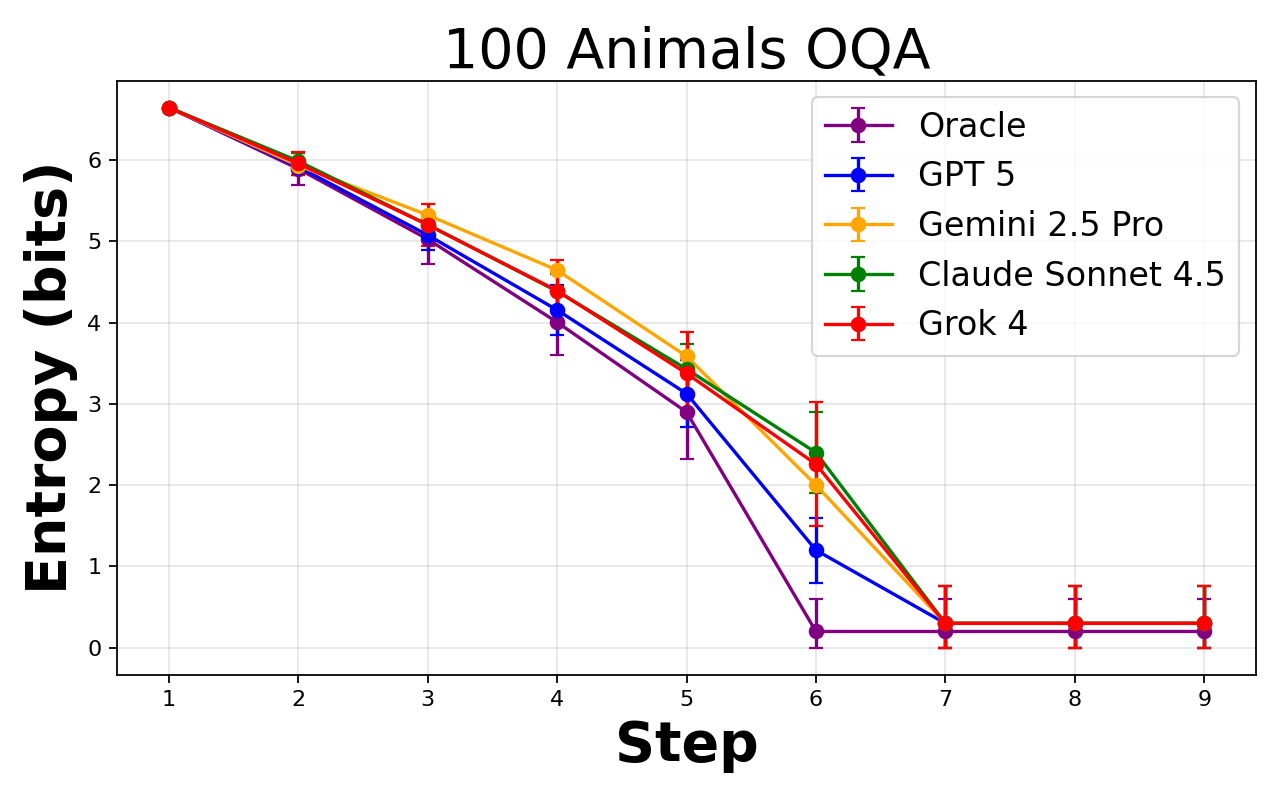}
%
\caption{Binary OQA with $N=100$. Curves show mean entropy $H_t=\log_2|C_t|$ across uniformly sampled targets, with $\pm 1$ standard deviation, along with the DP oracle. In Animals, duplicate attribute vectors can force stopping with a small residual entropy.}
\label{fig:binary-oqa-100}
\end{figure}

Binary OQA starts with a set $\mathcal{X}$ of $N$ items. Each item has a fixed yes or no value for each of the same attribute names, which can be viewed as its attribute vector. For example, in an animals table the item ``cat'' might have values like ``mammal = yes,'' ``has\_fur = yes,'' and ``can\_swim = no.'' One item $x^\star$ is chosen uniformly at random as the hidden target, and the candidate set starts at $C_0=\mathcal{X}$. Each turn asks for one attribute value and removes all items that disagree with the replies so far, leaving the updated candidate set $C_t$. With a uniform prior and noiseless answers, the belief is uniform on $C_t$. The uncertainty is $H_t=\log_2|C_t|$ and per question progress is $\Delta H_t = H_t - H_{t+1}$. A run ends when $|C_t|=1$, or earlier when no attribute can split $C_t$. The latter happens when the remaining items share the same full attribute vector, so some runs can end with $H_t>0$.

Let $T_{\text{model}}(x^\star)$ be the model's count and $T_{\text{DP}}(x^\star)$ the DP oracle's count under the same attribute menu and stopping rule. We define the metric ``planning gap" as the average number of extra questions the model uses compared to the oracle,
\(\mathbb{E}_{x^\star\sim\mathrm{Unif}(\mathcal{X})}\!\left[T_{\text{model}}(x^\star)-T_{\text{DP}}(x^\star)\right].\) A simple baseline is a greedy strategy that picks the attribute with the largest expected reduction in uncertainty. For a candidate set $C$ and a yes or no attribute $j$, asking $j$ splits $C$ into $C^{\text{yes}}$ and $C^{\text{no}}$. Near-balanced splits tend to be most informative. In this deterministic setting, the expected entropy reduction equals mutual information up to the log base, but it is optimal only for the next step and can be suboptimal over multiple steps.

Our DP oracle is optimal in the sense that it minimizes the expected number of questions until stopping. Let $\mathrm{Cost}(C)$ be the optimal expected number of additional questions when the remaining candidate set is $C$, assuming the hidden target is uniform on $C$. If an attribute $j$ splits $C$ into $C^{\text{yes}}$ and $C^{\text{no}}$, then asking $j$ costs one question plus the expected remaining cost weighted by branch sizes, giving $\mathrm{Cost}(C) =$
\begin{equation}
\min_{j:\ C^{\text{yes}},C^{\text{no}}\neq\emptyset}
\left(
1+\frac{|C^{\text{yes}}|}{|C|}\mathrm{Cost}(C^{\text{yes}})
+\frac{|C^{\text{no}}|}{|C|}\mathrm{Cost}(C^{\text{no}})
\right),
\label{eq:binary-dp}
\end{equation}
with $\mathrm{Cost}(C)=0$ when $|C|\le 1$ or when no attribute can split $C$.

\begin{theorem}[Optimality of the deterministic OQA oracle]\label{thm:oqa-oracle-optimal}
Consider any finite deterministic OQA game with candidate set $C$, a uniform target within $C$, a finite menu of allowed queries, and unit query cost. Each query $a$ partitions $C$ into nonempty cells $\{C_o(a):o\in\mathcal{O}_a\}$, and terminal sets are those with one candidate or no query that can split them. Let
\begin{equation}
V(C)=
\begin{cases}
0, & C \text{ terminal},\\[2mm]
\displaystyle \min_{a}\left(1+\sum_{o\in\mathcal{O}_a}\frac{|C_o(a)|}{|C|}V(C_o(a))\right), & \text{otherwise},
\end{cases}
\label{eq:generic-oqa-dp}
\end{equation}
where the minimum ranges over queries that split $C$. Then $V(C)$ is the minimum expected number of additional queries over all allowed adaptive decision trees. Any policy that attains the minimum in \eqref{eq:generic-oqa-dp} at every reachable set is optimal.
\end{theorem}

\begin{proof}
The proof is by induction on $|C|$. Terminal sets have value zero by definition. For nonterminal $C$, any valid decision tree must choose some first query $a$, pay one unit of cost, and then continue independently inside the cell $C_o(a)$ selected by the truthful answer. Since the target is uniform on $C$, that cell occurs with probability $|C_o(a)|/|C|$. By the induction hypothesis, the best possible continuation cost in each branch is $V(C_o(a))$. Thus every decision tree has expected cost at least the right side of \eqref{eq:generic-oqa-dp}. Choosing a minimizing query and then using optimal subtrees for all child cells attains this bound, so the recurrence is both necessary and sufficient.
\end{proof}

We compute $\mathrm{Cost}(C)$ with caching, since the same candidate set can arise from different question sequences. To plot an oracle trace for a fixed target $x^\star$, start from $C_0=\mathcal{X}$. At turn $t$, choose an attribute that minimizes \eqref{eq:binary-dp}, filter to obtain $C_{t+1}$, and record $H_t=\log_2|C_t|$. We run this on Places, Cars, and Animals with $N\in\{25,100\}$, using one API call per question with \texttt{temperature=0}, fresh sessions per target, and no external tools. Figures~\ref{fig:binary-oqa-25} and \ref{fig:binary-oqa-100} show mean entropy trajectories over uniformly sampled targets, with $\pm 1$ standard deviation, alongside the DP oracle.

\textbf{Summary and takeaways.} We evaluate frontier LLMs on 25- and 100-tier binary OQA datasets, tracking posterior entropy exactly and comparing each model to a dynamic programming (DP) oracle under the same attribute menu and stopping rule. However, even in a clean setting where uncertainty can be measured exactly, models leave efficiency on the table. They do make steady progress, but they still ask avoidable questions compared to an optimal strategy.

\section{Multiway Categorical OQA Experiments: 100, 200, and 300 Candidates}
\label{sec:kary-oqa}

Similar to binary OQA, multiway categorical OQA starts with a set $\mathcal{X}$ of $N$ items. Each item has a fixed categorical value for each of the same attribute names, which can be viewed as its attribute vector. For example, an item might have values like \texttt{color = red}, \texttt{shape = hexagon}, and \texttt{material = steel}. One item $x^\star$ is chosen uniformly at random as the hidden target, and the candidate set starts at $C_0=\mathcal{X}$. Each turn asks for the value of one attribute and removes all items that disagree with the replies so far, leaving the updated candidate set $C_t$. With a uniform prior and deterministic answers, the belief is uniform on $C_t$. As before, the uncertainty is $H_t=\log_2|C_t|$ and per question progress is $\Delta H_t = H_t - H_{t+1}$. A run ends when $|C_t|=1$, or earlier when no attribute can split $C_t$.

Formally, let $\mathcal{A}$ be the finite attribute set, and for each $a\in\mathcal{A}$ let $a:\mathcal{X}\to\mathcal{V}_a$ map items to a finite value set. At turn $t$ the agent chooses $a_t\in\mathcal{A}$, observes the truthful reply $o_t=a_t(x^\star)$, and filters by consistency, $C_{t+1}=\{x\in C_t:\ a_t(x)=o_t\}$. We call a tier $k$-ary, where $k=\max_{a\in\mathcal{A}}|\mathcal{V}_a|$. In our released tiers, $k=5$. We keep the early stopping rule above for completeness, although in the released tiers the attribute vectors are unique, so runs end with $|C_t|=1$. As in the binary setting, we summarize planning by questions to stop. Let $T_{\text{model}}(x^\star)$ be the model's count and $T_{\text{DP}}(x^\star)$ the DP oracle's count under the same attribute menu and stopping rule. The planning gap is $\mathbb{E}_{x^\star\sim\mathrm{Unif}(\mathcal{X})}\!\left[T_{\text{model}}(x^\star)-T_{\text{DP}}(x^\star)\right]$.

A simple baseline is a greedy strategy that picks the attribute with the largest expected reduction in uncertainty. Fix a candidate set $C$ and an attribute $a$. For each value $v\in\mathcal{V}_a$, define $C_v=\{x\in C:\ a(x)=v\}$. Under the uniform belief on $C$, $\Pr[o=v]=|C_v|/|C|$, and the expected entropy reduction from asking $a$ is $\mathrm{EIG}(a;C)=\log_2|C|-\sum_{v\in\mathcal{V}_a:\,|C_v|>0}\frac{|C_v|}{|C|}\log_2|C_v|.$

Near-balanced partitions tend to be most informative. In this deterministic setting, $\mathrm{EIG}$ matches mutual information up to the log base (Prop.~\ref{prop:kl-mi}), but it is optimal only for the next step and can be suboptimal over multiple steps.

The DP oracle is optimal under our rules, in the sense that it minimizes the expected number of questions until stopping. Let $\mathrm{Cost}(C)$ be the optimal expected number of additional questions when the remaining candidate set is $C$, assuming the hidden target is uniform on $C$. For an attribute $a$, let $B_a(C)=\{v\in\mathcal{V}_a:\ |C_v|>0\}$. If $|B_a(C)|=1$, then $a$ does not split $C$ and is ignored. Otherwise, $\mathrm{Cost}(C)=$
\begin{equation}
\min_{a\in\mathcal{A}:\,|B_a(C)|\ge 2}\left(1+\sum_{v\in\mathcal{V}_a:\,|C_v|>0}\frac{|C_v|}{|C|}\,\mathrm{Cost}(C_v)\right),
\label{eq:kary-dp}
\end{equation}
with $\mathrm{Cost}(C)=0$ when $|C|\le 1$ or when no attribute can split $C$. We compute $\mathrm{Cost}(C)$ with caching using a canonical representation of $C$, such as a sorted tuple of item IDs. To plot an oracle trace for a fixed target $x^\star$, start from $C_0=\mathcal{X}$. At turn $t$, choose an attribute that minimizes \eqref{eq:kary-dp}, apply $C_{t+1}=\{x\in C_t:\ a_t(x)=o_t\}$ using the target's true value, and record $H_t=\log_2|C_t|$ until stopping.

We use three tiers with $|\mathcal{X}|\in\{100,200,300\}$. All tiers share the same eight attributes, namely \texttt{color}, \texttt{shape}, \texttt{material}, \texttt{size}, \texttt{pattern}, \texttt{origin}, \texttt{use\_case}, and \texttt{energy}. Each query must request the value of exactly one named attribute, and the evaluator answers by table lookup. 

We evaluate $30$ targets per tier with tool use disabled, using one API call per question with \texttt{temperature=0} in a fresh session per target. Mean entropy trajectories are plotted with $\pm 1$ standard deviation alongside the DP oracle in \Cref{fig:kary-oqa}.

\paragraph{Summary and takeaways.} We extend OQA to multiway categorical attributes using deterministic table lookups, with 100, 200, and 300 tiers and a fixed menu of eight attributes. We again track posterior entropy and measure efficiency by planning gap. Frontier models consistently shrink the consistent set over turns, and yet they remain less question-efficient than the dynamic programming (DP) oracle.

\begin{figure}[!htbp]
\centering
\IfFileExists{figs/kary100_entropy_plot.png}{
  \includegraphics[width=0.65\linewidth]{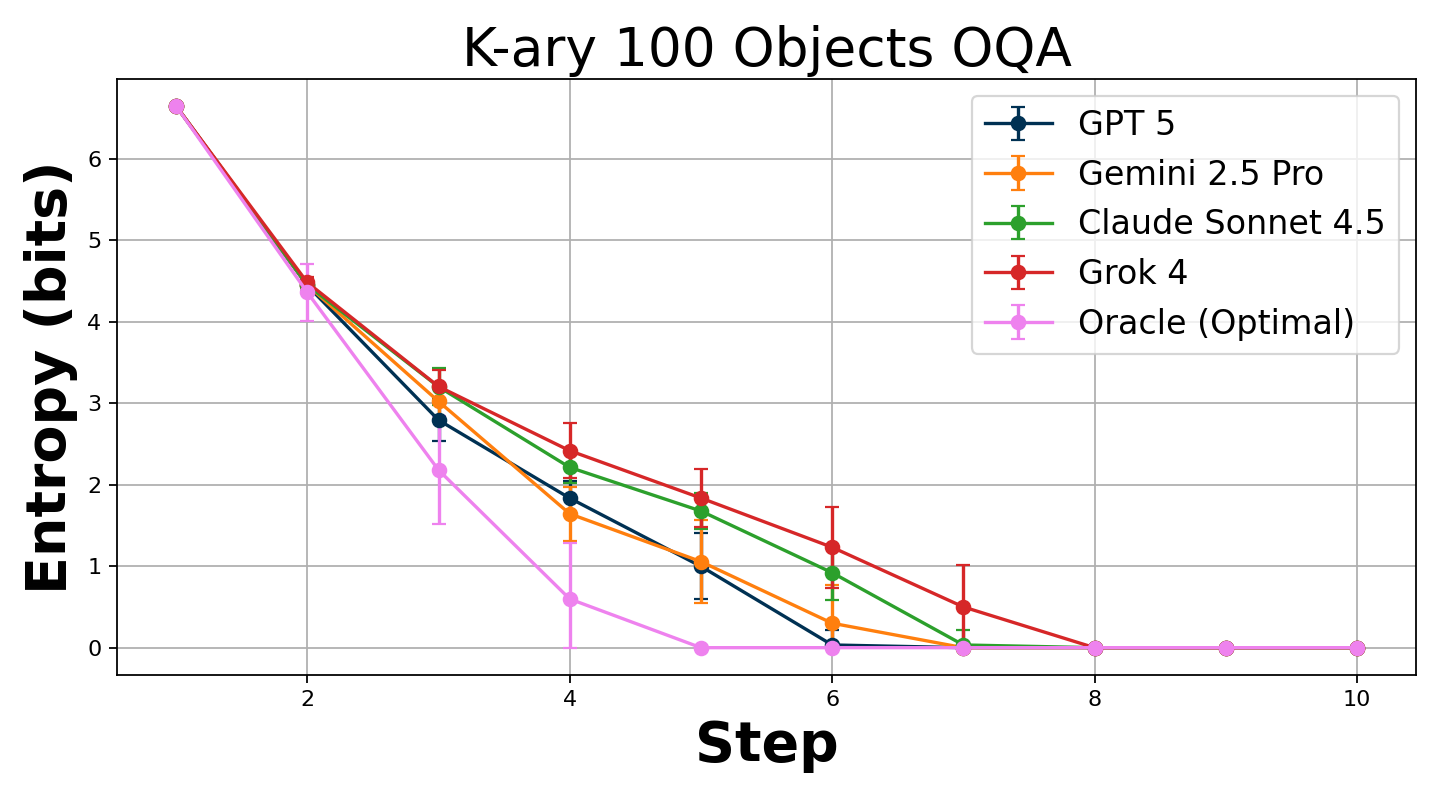}
}{
  \fbox{\parbox[c][0.80in][c]{0.95\linewidth}{\centering\small Multiway-100 curves (placeholder)}}
}
\vspace{0.03in}

\IfFileExists{figs/kary200_entropy_plot.png}{
  \includegraphics[width=0.65\linewidth]{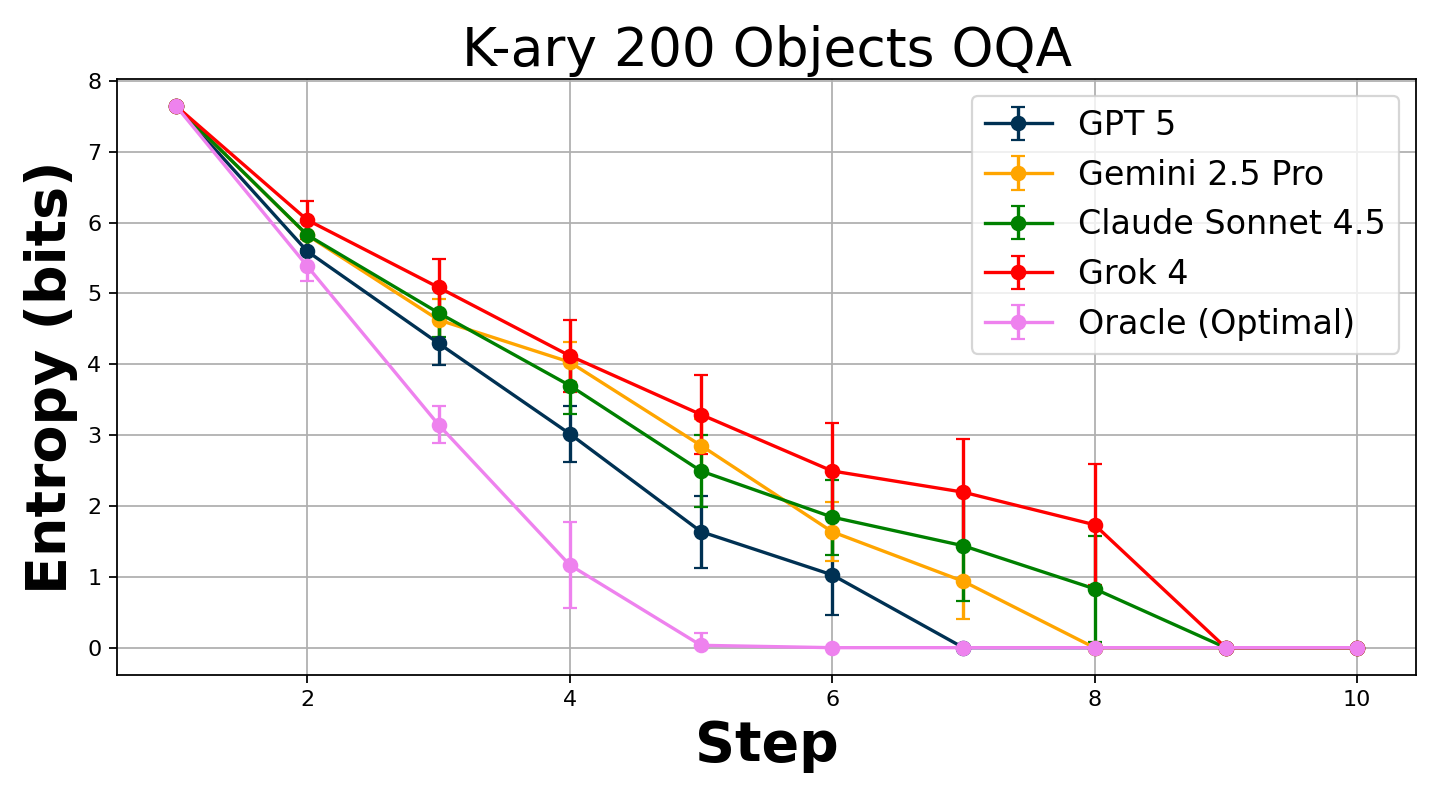}
}{
  \fbox{\parbox[c][0.85in][c]{0.95\linewidth}{\centering\small Multiway-200 curves (placeholder)}}
}
\vspace{0.03in}

\IfFileExists{figs/kary300_entropy_plot.png}{
  \includegraphics[width=0.65\linewidth]{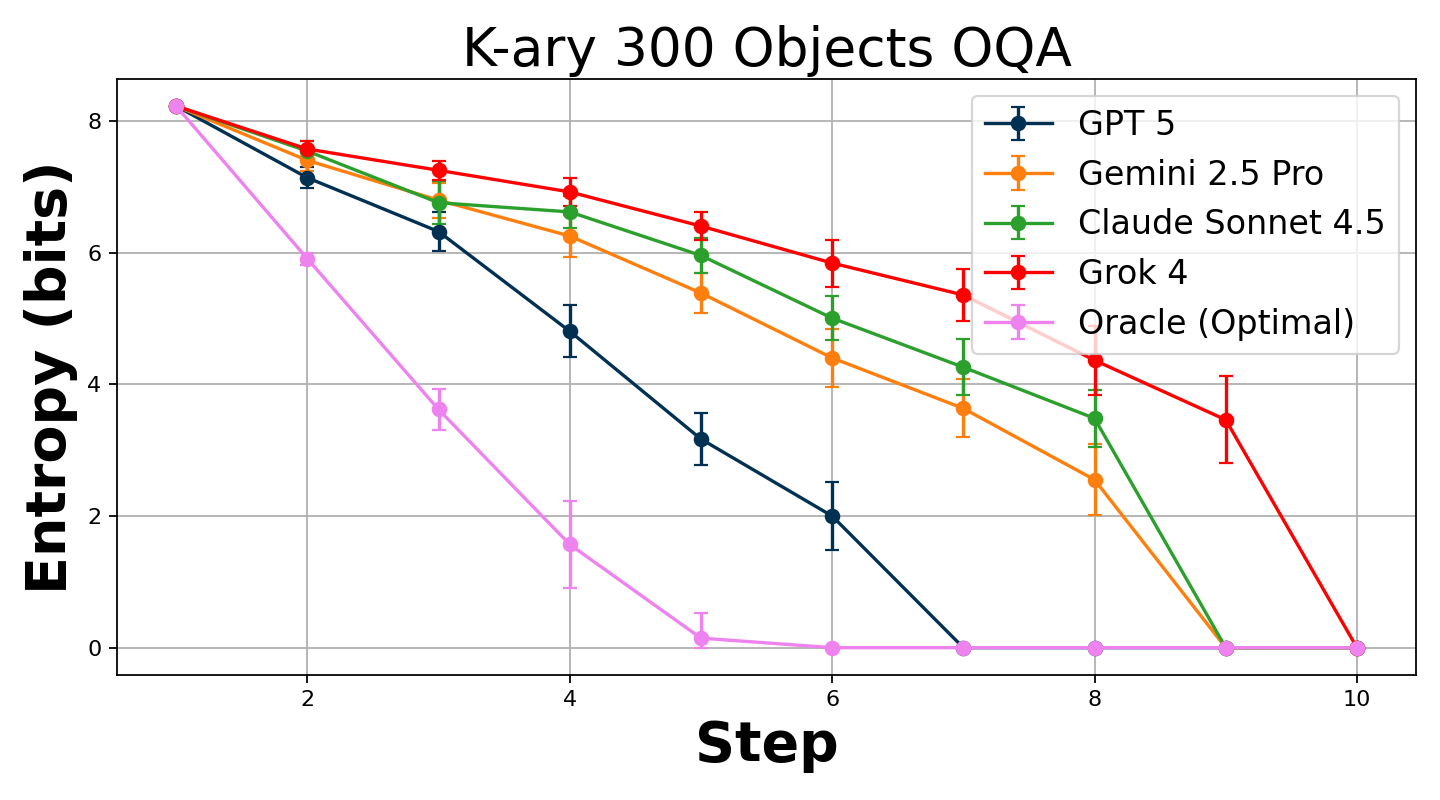}
}{
  \fbox{\parbox[c][0.85in][c]{0.95\linewidth}{\centering\small Multiway-300 curves (placeholder)}}
}
\caption{Multiway categorical OQA for $|\mathcal{X}|\in\{100,200,300\}$ (top to bottom). Curves show mean posterior entropy over uniformly sampled targets. The DP oracle uses the same query menu and stopping rule. Error bars show $\pm 1$ standard deviation when available.}

\label{fig:kary-oqa}
\vspace{-0.11in}
\end{figure}

\section{Prompt Autocompletion as Active Inference}
\label{sec:prompt-autocomplete}

This is our first prompt experiment. Suppose a user asks for a product description but does not specify a style. Each task provides a product name and three feature phrases. We model the desired style as a hidden variable \(U=(\text{tone},\text{length},\text{format})\). The tone is either formal or friendly, length is either short or medium, and format is either bullet points or a paragraph. We sample these three style choices independently and uniformly. The agent may ask up to three clarifying questions before writing one final description. In this experiment, the questions use fixed templates, and we simulate the answers from the hidden \(U\). This keeps the focus on the agent's choice of what to ask, rather than how the questions are phrased. Each question targets exactly one component of \(U\), and we never ask about the same component twice.

We assume that the answers are always truthful, and we add uncertainty only when updating the agent's belief about \(U\), using a symmetric label noise model, so the posterior does not become overly confident. We also charge a synthetic token cost for each clarification. The first costs \(c_1=24\) tokens, then \(c_2=48\), then \(c_3=72\). After any clarifications, we call the base model once and request a description using the MAP style \(\hat U\). We further assume that the agent's belief over styles factorizes across the three attributes, so \(q_t(U)=q_t(\text{tone})\,q_t(\text{length})\,q_t(\text{format})\), and we start from a uniform prior. When the agent asks about one attribute \(v\) (tone, length, or format), it updates only \(q_t(v)\). We model mistakes in this update with a fixed error rate \(\varepsilon=0.12\). The information gained from that question is \(\Delta I_t(v)=\mathrm{KL}(q_{t+1}(v)\,\|\,q_t(v))\), measured in bits.

After the agent produces the final description, a deterministic verifier checks it. The verifier requires each feature phrase to appear as an exact substring, and it checks format and word count. A short description has at most 70 words, while a medium description has 70 to 160 words. Tone is not checked, so the sampled style and the verifier reward do not perfectly match. We compare five question asking policies. \texttt{baseline} asks no questions and uses the default style. \texttt{ask\_all} asks about tone, length, and format. \texttt{random} samples \(K\) uniformly from \(\{0,1,2,3\}\), then asks \(K\) distinct attributes in random order. \texttt{active} chooses the unasked attribute with the largest entropy per expected token cost, \(H(q_t(v))/c_{|\mathcal{C}|+1}\), and stops when this score falls below \(\epsilon\). \texttt{active\_weighted} uses the same rule, but weights the entropy term by \(w_v\), with \(w_{\text{format}}=1.0\), \(w_{\text{length}}=0.6\), and \(w_{\text{tone}}=0.25\).

In \Cref{fig:prompt-autocomplete}, we fix \(\epsilon=0.02\) and \(K_{\max}=3\) on 48 tasks to show the cost accuracy trade-off across policies. We then tune \texttt{active\_weighted} via grid search over \(\epsilon\in\{0,0.005,0.01,0.02,0.04\}\) and \(K_{\max}\in\{1,2,3\}\), maximizing \(J(\epsilon,K_{\max})=\mathbb{E}[\mathrm{ok}]-0.02\,\mathbb{E}[\text{tokens\_total}]/1000\), where \(\text{tokens\_total}\) is the model-reported tokens for the final generation plus the synthetic clarification costs. We score on 24 training tasks and evaluate on 24 held out tasks. The best setting is \(\epsilon^\star=0.01\) and \(K_{\max}^\star=2\), reaching \(0.375\) compliance at about \(219\) tokens per task (vs.\ \texttt{baseline} at \(0.0417\) compliance and about \(112\) tokens).

\textbf{Summary and takeaways.} We study how to spend a limited token budget between clarifying turns and a final completion when the style is hidden, answers to fixed clarification templates are truthful, and when success is measured by a deterministic verifier. Policies that ask targeted clarifications improve verifier pass rates relative to asking none, showing that paying for a small number of well chosen questions can be worthwhile, but only when those questions buy information the final output will actually be judged on.
\begin{figure}[!htbp]
\centering
\subfigure[Clarifications used]{%
  \IfFileExists{figs/prompt_clarifications_hist_updated.png}{%
    \includegraphics[width=0.48\linewidth]{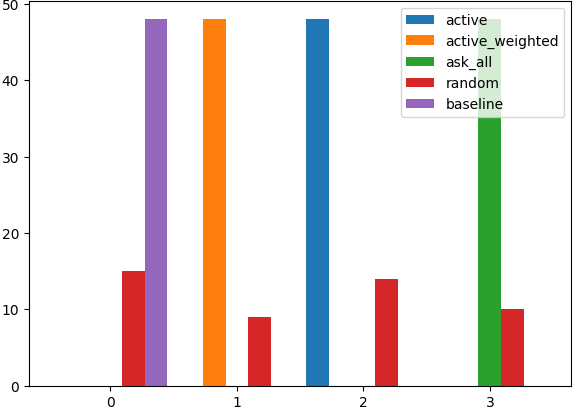}%
  }{%
    \fbox{\parbox[c][0.45in][c]{0.45\linewidth}{\centering\small
    Placeholder for clarifications histogram.}}%
  }}\hfill
\subfigure[Total token ECDF]{%
  \IfFileExists{figs/prompt_token_ecdf_updated.png}{%
    \includegraphics[width=0.48\linewidth]{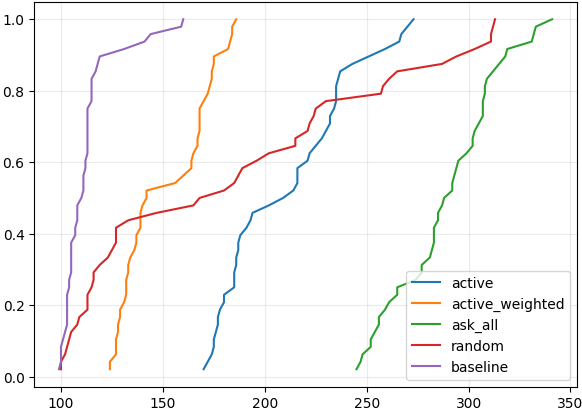}%
  }{%
    \fbox{\parbox[c][0.45in][c]{0.45\linewidth}{\centering\small
    Placeholder for token ECDF.}}%
  }}
\vspace{0.04in}

\subfigure[Compliance versus average tokens]{%
  \IfFileExists{figs/prompt_efficiency_frontier_largefonts.png}{%
    \includegraphics[width=0.68\linewidth]{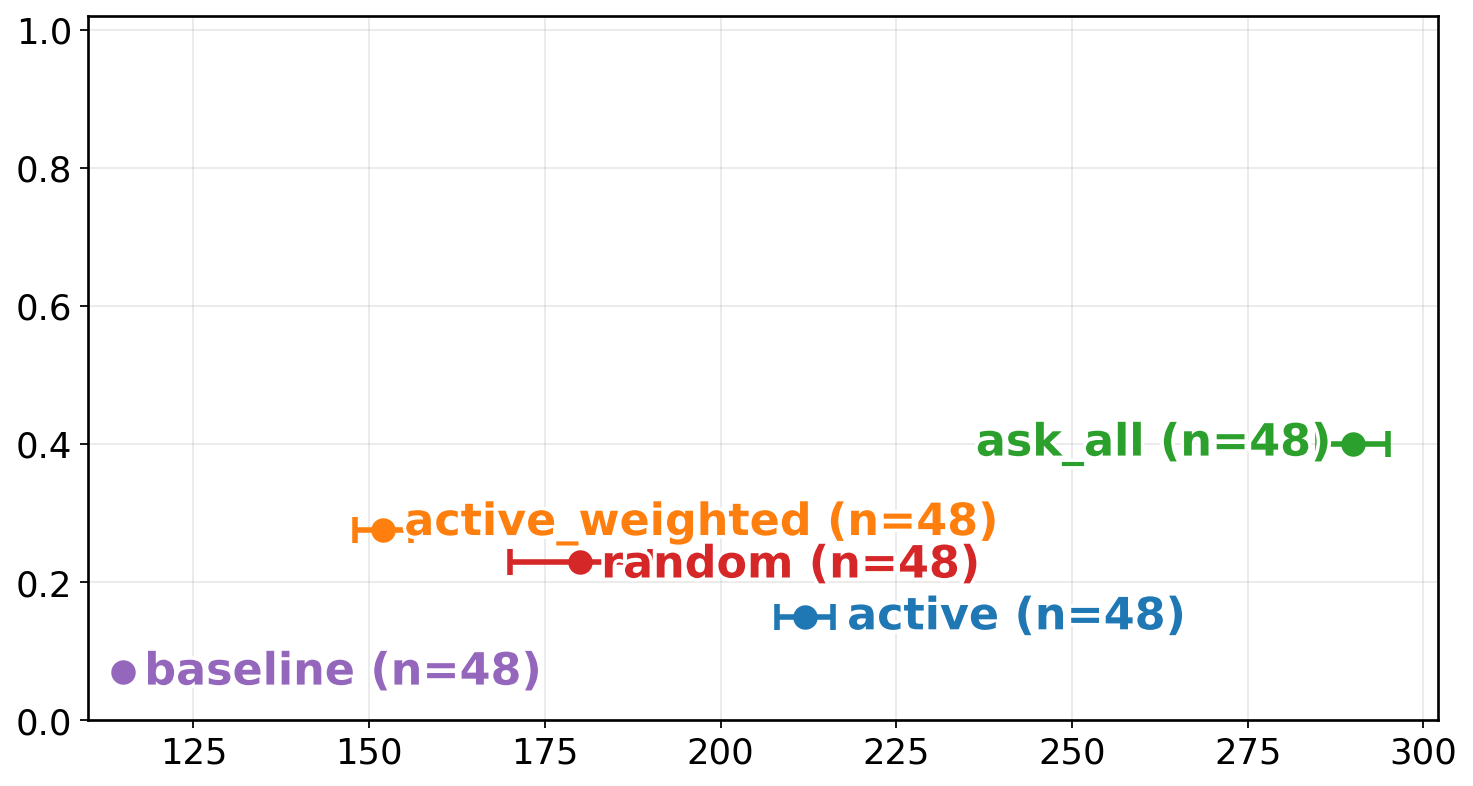}%
  }{%
    \fbox{\parbox[c][0.45in][c]{0.65\linewidth}{\centering\small
    Placeholder for efficiency frontier.}}%
  }}
\vspace{-0.03in}
\caption{\small Prompt autocompletion on 48 synthetic product tasks with $\epsilon=0.02$ and $K_{\max}=3$. Panels show clarification counts by policy, the ECDF of total tokens per task, and verifier pass rate versus average tokens. \texttt{active\_weighted} improves compliance at a modest token cost relative to \texttt{baseline} and \texttt{random}, while \texttt{ask\_all} spends the most tokens.}
\label{fig:prompt-autocomplete}
\vspace{-0.04in}
\end{figure}

\section{Automated Prompt Optimization}
\label{sec:prompt-optimization}
In this second prompt experiment, the model never asks the user follow-up questions, but rather treats the choice
of system prompt as an online decision problem under a fixed token budget, aiming to quickly
identify which prompt variant yields the most reliable multiple-choice answers. We fix a small
library $P$ of candidate system prompts: $\mathrm{letter\_only}$, $\mathrm{short\_reasoning}$,
$\mathrm{eliminate\_two}$, $\mathrm{keyword\_match}$, $\mathrm{units\_and\_scales}$, and
$\mathrm{contrastive\_explanations}$. At training step $t$ we select a prompt $p_t\in P$, format one
ARC-Challenge question \cite{clark2018think} using a fixed user-message template, and query the base
model once with \texttt{temperature=0} so the output is deterministic given the prompt. From the
response we extract the first standalone letter in $\{\texttt{A},\texttt{B},\texttt{C},\texttt{D}\}$,
score it as correct or incorrect to obtain $y_t\in\{0,1\}$, and record the token cost $\tau_t$.

Each prompt $p$ is modeled as a Bernoulli arm with unknown accuracy $\theta_p$, and we maintain
independent Beta posteriors $\theta_p\sim\mathrm{Beta}(\alpha_p,\beta_p)$ initialized at
$\alpha_p=\beta_p=1$. After observing $y_t$ for the chosen prompt $p_t$, we update only that prompt
via $\alpha_{p_t}\leftarrow\alpha_{p_t}+y_t$ and $\beta_{p_t}\leftarrow\beta_{p_t}+1-y_t$, and
report the posterior mean $\mathbb{E}[\theta_p]=\alpha_p/(\alpha_p+\beta_p)$.

To summarize uncertainty about which prompt is truly best, we define
$P^\star=\arg\max_p \theta_p$ and estimate the induced distribution over $P^\star$ by Monte Carlo
sampling, where we draw one $\theta_p$ from each Beta factor and take the maximizing prompt for each
draw. From this empirical distribution we compute an entropy $H(P^\star)$ in bits, and after each
training step we log the realized entropy drop
$\Delta H_t=\max\{H_{\text{before}}-H_{\text{after}},0\}$ along with an information-efficiency
measure \(\mathrm{IE}_{1\mathrm{k}}(t)=1000\,\Delta H_t/\tau_t,\)
which can be read as the entropy reduction achieved per thousand tokens spent on that step. 

Prompt selection is handled by an outer-loop policy that chooses $p_t$ each step. Round robin cycles through prompts, Thompson sampling draws $\tilde\theta_p$ from each posterior and picks the largest, KG chooses the prompt with the highest expected one-step gain in the current best posterior-mean accuracy, MI chooses the prompt with the largest expected reduction in $H(P^\star)$, and EFE blends MI and KG with a linearly decaying epistemic weight to shift from exploration to exploitation. Since prompts can differ in response length, MI, KG, and EFE normalize their scores by a per-prompt
token-cost estimate tracked with an exponential moving average of $\tau_t$. Training stops at 400 questions or $B=120{,}000$ tokens. In \Cref{fig:promptopt} the 400-question cap binds (about $5\times10^4$ tokens), after which we select the prompt with the highest posterior mean and evaluate
it on a shared validation subset. KG tends to concentrate queries early, leaving rarely tested prompts near the 0.5 prior mean, while MI and EFE are often most informative early and Thompson sampling and round robin spread learning more evenly.

\begin{figure}[!htbp]
\centering
\subfigure[Posterior mean accuracies]{%
  \IfFileExists{figs/promptopt_posteriors.png}{%
    \includegraphics[width=0.48\linewidth]{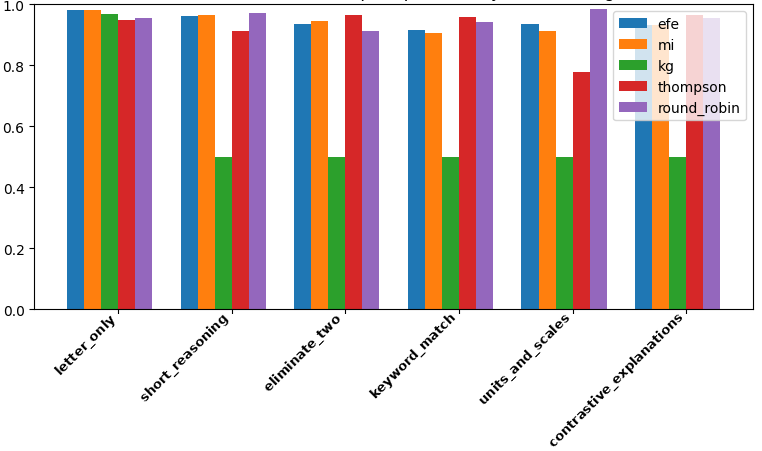}%
  }{%
    \fbox{\parbox[c][0.55in][c]{0.45\linewidth}{\centering\small
    Placeholder for posterior accuracy bar plot.}}%
  }}\hfill
\subfigure[Information gain per 1k tokens]{%
  \IfFileExists{figs/promptopt_ie_updated.png}{%
    \includegraphics[width=0.48\linewidth]{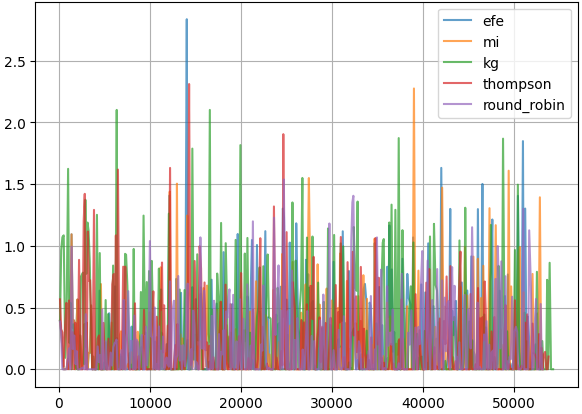}%
  }{%
    \fbox{\parbox[c][0.55in][c]{0.45\linewidth}{\centering\small
    Placeholder for information efficiency plot.}}%
  }}
\vspace{0.04in}

\subfigure[Entropy of best-prompt identity]{%
  \IfFileExists{figs/prompt_opt_uncertainity.png}{%
    \includegraphics[width=0.68\linewidth]{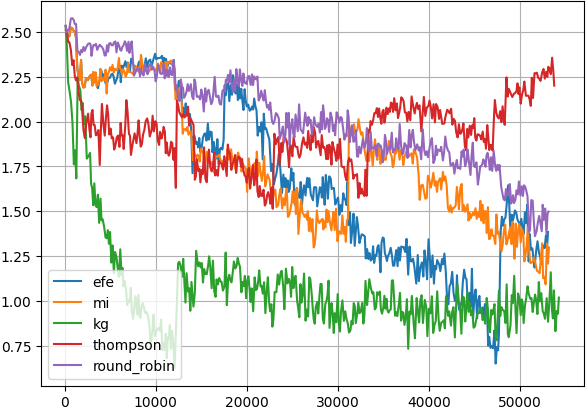}%
  }{%
    \fbox{\parbox[c][0.55in][c]{0.65\linewidth}{\centering\small
    Placeholder for entropy over best prompt.}}%
  }}
\vspace{-0.03in}
\caption{\small Automated prompt optimization under a token budget. Panels show final posterior mean accuracies, realized information gain per 1{,}000 tokens, and posterior entropy over the identity of the best prompt. Policies that reduce uncertainty early can avoid committing too soon.}
\label{fig:promptopt}
\vspace{-0.04in}
\end{figure}

\textbf{Summary and takeaways.} We treat system-prompt choice as an online decision problem under a fixed token budget, using repeated, automatically scored trials to update which prompt variant looks best. The main lesson is to use early trials to reduce uncertainty rather than commit too soon, then choose the variant with the strongest track record under the same scoring rule used throughout training.

\section{Limitations and Discussion}
\label{sec:limitations}
This paper intentionally studies restricted forms of context acquisition. In OQA, the AI assistant must identify a hidden item by asking about entries in a fixed attribute table. We test both yes--no tables and tables where answers take a small set of named values (Sections~\ref{sec:binary-oqa} and \ref{sec:kary-oqa}). The advantage is precision: after each question, we can say exactly how many items remain valid, and we can compare the model to a best-possible strategy under the same rules. The cost is realism. Real conversations may include vague replies, new constraints, contradictory preferences, missing files, and evidence that does not fit a prewritten attribute list. A stricter future benchmark would allow free-form questions and ground answers in richer inputs, including images and compositional attributes, for example, CLEVR-style scenes \citep{johnson2017clevr}.

The current experiments cover three real attribute tables (Places, Cars, Animals), larger synthetic multiway attribute tables up to 300 items, and two token-budget studies on how to spend limited tokens (Sections~\ref{sec:prompt-autocomplete} and \ref{sec:prompt-optimization}). This is enough to test the central accounting principle, but not enough to claim a complete agent benchmark. We leave out temporal reasoning, multimodal perception, robotics, open-ended tool-use, and multi-user collaboration, all of which make context acquisition harder \citep{mirza2016scene, lanillos2021active}.

We also turn off tools on purpose. We disable retrieval, function calls, and scratchpads so we can measure what the model can do on its own \citep{lewis2020rag,schick2023toolformer,nye2021show}. In real products, tools do a lot of the heavy lifting. They can store the candidate list, filter it exactly, keep memory across turns, and call external systems. A tool-enabled version of the benchmark would test whether the agent asks for the right missing information before invoking a tool, instead of hallucinating missing arguments or making premature calls.

There is also a scaling issue with the oracle we use for comparison. Our dynamic program gives the best strategy under our rules, but as the number of items grows, the number of different remaining sets that can appear may explode. It is manageable for the tiers we release, but it will not stay manageable forever. For larger tasks, approximate baselines based on sampling, search, or learned lookahead will be needed \citep{kirsch2019batchbald,schrittwieser2020mastering}. More broadly, a locally good question is not always part of the best overall plan. This is why we report a DP oracle where possible rather than treating greedy information gain as the final answer.

The two token-budget studies are controlled demonstrations, not full replicas of real workflows. In prompt autocompletion (Section~\ref{sec:prompt-autocomplete}), the assistant chooses among fixed clarifications, and answers are simulated from a hidden style choice. This isolates whether the assistant asks about the right variables, but it does not test natural phrasing, unclear answers, user impatience, or changing preferences. In automated prompt optimization (Section~\ref{sec:prompt-optimization}), we pick from a small handwritten set of prompt variants and test them on a multiple-choice benchmark. This is clean and automatically scored, but real prompt libraries are larger, correlated, task-dependent, and often judged by softer preferences. Token costs can also change across deployments. Since we evaluate frontier models through hosted APIs, model updates can affect reproducibility.

Finally, efficient context acquisition is dual use. The same machinery that helps an assistant ask fewer and better questions can help a malicious actor collect sensitive details, profile users, or probe hidden system rules. We discuss these connections in the appendix, but we do not test attacks or defenses. Deployment should pair information-gain objectives with limits on sensitive follow-ups, rate limits, audit trails, consent for personal data collection, and privacy-aware safeguards \citep{huang2024lifting,charles2024fine,freiberger2025you}.

\section{Conclusion and Future Work}
\label{sec:conclusion}

We presented active inference as a context-acquisition layer for AI agents. The central decision is whether to act with the current context or spend tokens, interaction, or computation to acquire more evidence. The bilevel formulation separates the posterior update caused by a hypothetical observation from the outer decision that chooses the next context action or task action. OQA makes this decision measurable through exact posteriors and a dynamic programming oracle. Across tiers, frontier models reduce uncertainty each turn but still ask more questions than the oracle. The two token-budget studies show the same principle in prompting: a small amount of targeted clarification or prompt experimentation can improve final success when it buys decision-relevant information.

Future work should move beyond fixed tables to noisy, open-ended, and multimodal interaction. It should also scale oracle baselines with sampling-based search, evaluate tool-enabled agents with memory, and integrate explicit stopping decisions. A central open direction is to make context acquisition safe by design: the agent should learn what it needs to know, but not overcollect information that is irrelevant, sensitive, or unsafe to request.

\section*{Impact Statement}

This paper advances machine learning methods for deciding when an interactive system should acquire more context versus act, using active inference, information gain, and controlled oracle diagnostics. Potential benefits include fewer unnecessary questioning turns, lower token and compute cost, better tool-use reliability, and improved task success under ambiguity.

Risks are that more efficient context acquisition can be misused to extract sensitive information, enable profiling, or speed up adversarial probing of system defenses. OQA itself uses synthetic data and simulated answers, so it does not require collecting personal data. In deployment, mitigation includes limiting and auditing follow-ups, obtaining consent for personal data collection, and adding explicit privacy and safety constraints to the objective.

\FloatBarrier
\bibliographystyle{plainnat}
\bibliography{custom}


\appendix

\section{Optimal Question Asking: API Transcripts and Prompt Templates}\label{app:oqa-api}

OQA is scored from the transcript alone. At each turn, the model asks about exactly one named attribute. The oracle returns the answer by looking it up in a finite table of attributes. After the reply, a scoring script filters the remaining set $C_t$ by exact consistency. Because every step is deterministic, the score is fully reproducible. With the same prompt and the same hidden target, replaying the same message sequence yields the same $C_t$, the same entropy $H_t=\log_2|C_t|$, and the same stopping time.

The protocol enforces two rules. Every question must use an attribute name from a fixed list. Each turn may refer to only one attribute. In binary OQA, the oracle replies \texttt{"Yes"} or \texttt{"No"}, and the update is $C_{t+1}=\{x\in C_t \mid z_j(x)=o_t\}$. In multiway OQA, the oracle replies with a single value string, and the update is $C_{t+1}=\{x\in C_t \mid a_t(x)=o_t\}$. A run ends when $|C_t|=1$. It can also end earlier when the remaining items share the same vector of attributes, so no allowed query can split them further. For reliable parsing, a strict JSON format is available. The next query is read from the \texttt{"question"} field, while any extra fields are ignored and may be used for debugging.

\begin{oqacode}{orange}{Binary OQA user prompt template (plain-text mode)}
This is a <DATASET> attributes dataset. I have a hidden <OBJECT TYPE> in mind.

Ask only yes/no questions about exactly one attribute at a time, chosen from:
<ATTRIBUTE_1>, <ATTRIBUTE_2>, ..., <ATTRIBUTE_d>.

After each answer, do the following:
1) Show your current belief over the remaining <OBJECT TYPE>s.
2) Report how many candidates are left.
3) If you can estimate entropy mentally, include it.

Do not run code or use tools.

Stop when only one candidate remains, or when multiple items share an identical
attribute vector and cannot be separated.
\end{oqacode}

\begin{oqafig}{blue}{Binary OQA placeholder instantiations (domain families)}
\small
\centering
\begin{tabular}{lll}
\textbf{Domain family} & \textbf{\texttt{<DATASET>}} & \textbf{\texttt{<OBJECT TYPE>}} \\
\hline
Animals & \texttt{ANIMALS} & \texttt{animal} \\
Cars    & \texttt{CARS}    & \texttt{car} \\
Places  & \texttt{PLACES}  & \texttt{place} \\
\end{tabular}

\vspace{0.4em}
\noindent\small The concrete attribute list \texttt{<ATTRIBUTE\_1..d>} is provided in each dataset's
\texttt{data/attributes.txt}.
\end{oqafig}

\begin{oqafig}{teal}{Binary OQA sample transcript (2 turns shown)}
\small
\textbf{System.} You are a careful planner.\par
\textbf{User.} This is a \texttt{PLACES} attributes dataset. I have a hidden place in mind.
Ask only yes/no questions about one attribute at a time from the provided list.\par

\medskip
\textbf{Assistant (Turn 1).} \emph{Question:} Is it coastal?\par
\emph{Belief (before the answer):} uniform over all 25 places.\par
\emph{Uncertainty:} $H_0=\log_2 25 \approx 4.64$ bits. \emph{Stop:} no.\par

\smallskip
\textbf{User.} No.\par

\smallskip
\textbf{Assistant (Turn 2).} \emph{Update:} filter to candidates with \texttt{coastal = false}.\par
\emph{Remaining candidates:} 9 places (uniform over the remaining set).\par
\emph{Uncertainty:} $H_1=\log_2 9 \approx 3.17$ bits.\par
\emph{Next question:} Is it mountainous?\par

\smallskip
\textbf{User.} Yes.\par
\end{oqafig}

\begin{oqacode}{purple}{Multiway ($k$-ary) user prompt template (plain-text mode)}
This is a $k$-ary attributes dataset with <N> abstract objects. I have a hidden item in mind.

Ask questions about exactly one attribute at a time, chosen from this list:
color, shape, material, size, pattern, origin, use_case, energy.

The answer will be a single value for that attribute, for example "red" or "triangle".

After each answer, do the following:
1) Show your current belief over the remaining items.
2) Report how many candidates are left.
3) If you can estimate entropy mentally, include it.

Do not run code or use tools.

Stop when only one candidate remains, or when multiple items share an identical
attribute vector and cannot be separated.
\end{oqacode}

\begin{oqafig}{red}{Multiway ($k$-ary) sample transcript (2 turns shown)}
\small
\textbf{System.} You are a careful planner.\par
\textbf{User.} This is a $k$-ary attributes dataset with 200 abstract objects. Ask about exactly
one attribute per turn from: \texttt{color, shape, material, size, pattern, origin, use\_case, energy}.
Each answer will be a single value (e.g., \texttt{"red"}).\par

\medskip
\textbf{Assistant (Turn 1).} \emph{Question:} What is the \texttt{energy}?\par
\emph{Belief (before the answer):} uniform over all 200 items.\par
\emph{Uncertainty:} $H_0=\log_2 200 \approx 7.64$ bits. \emph{Stop:} no.\par

\smallskip
\textbf{User.} electric\par

\smallskip
\textbf{Assistant (Turn 2).} \emph{Update:} filter to candidates with \texttt{energy = electric}.\par
\emph{Remaining candidates:} (uniform over the filtered set).\par
\emph{Uncertainty:} $H_1 \approx 3.12$ bits.\par
\emph{Next question:} What is the \texttt{shape}?\par

\smallskip
\textbf{User.} hexagon\par
\end{oqafig}

\section{Exact Dynamic Programming Oracle}\label{app:oracle-dp}

The oracle curves in our figures are produced by an exact planner that minimizes the expected number of questions required to identify the hidden target under a uniform prior and deterministic answers. In these tiers, each reply removes every candidate that is inconsistent with it, so the posterior after any transcript is uniform over the remaining candidates. The belief state is therefore completely determined by the current candidate set $C\subseteq\mathcal{X}$. Asking one attribute partitions $C$ into subsets indexed by the possible replies, and the next belief is obtained by restricting to the subset that matches the observed reply. This reduces optimal question asking to a finite decision problem over sets, with termination when $|C|\le 1$ or when the allowed attributes cannot further separate the remaining items.

Let $C$ be the current candidate set. A query $q\in\mathcal{Q}$ induces a set of possible replies $\mathcal{R}(q,C)$ and corresponding next candidate sets $\{C_r\}_{r\in\mathcal{R}(q,C)}$, where
$C_r=\{x\in C \mid q(x)=r\}$.
Under a uniform posterior on $C$ and deterministic replies, $\Pr[r]=|C_r|/|C|$. Define $\mathrm{Cost}(C)$ as the optimal expected number of additional questions until termination, starting from $C$. A set is terminal when $|C|\le 1$ or when every query fails to split $C$ into at least two nonempty subsets. Bellman optimality gives the recursion
\begin{equation}
\mathrm{Cost}(C)=
\min_{\substack{q\in\mathcal{Q}\\|\mathcal{R}(q,C)|\ge 2}}
\left(
1+\sum_{r\in\mathcal{R}(q,C)}\frac{|C_r|}{|C|}\,\mathrm{Cost}(C_r)
\right),
\qquad
\mathrm{Cost}(C)=0\ \text{if $C$ is terminal}.
\label{eq:oqa-dp-general}
\end{equation}
The oracle policy at state $C$ is the minimizer in \eqref{eq:oqa-dp-general}. Greedy maximization of one step expected information gain optimizes only the immediate entropy drop, while \eqref{eq:oqa-dp-general} is optimal for the expected remaining number of questions.

In the binary tiers, each query is an index $j\in[d]$ with replies in $\{0,1\}$. Writing
$C^{\text{yes}}=\{x\in C \mid z_j(x)=1\}$ and $C^{\text{no}}=C\setminus C^{\text{yes}}$, the recursion becomes
\begin{equation}
\mathrm{Cost}(C)=
\min_{\substack{j\in[d]\\ C^{\text{yes}}\neq\emptyset\\ C^{\text{no}}\neq\emptyset}}
\left(
1+\frac{|C^{\text{yes}}|}{|C|}\mathrm{Cost}(C^{\text{yes}})
+\frac{|C^{\text{no}}|}{|C|}\mathrm{Cost}(C^{\text{no}})
\right),
\label{eq:oqa-dp-binary}
\end{equation}
with $\mathrm{Cost}(C)=0$ when $|C|\le 1$ or when no attribute splits $C$. This is the dynamic program used to construct the optimal expected depth decision tree for the binary oracle.

In the multiway categorical tiers, a query is an attribute $a\in\mathcal{A}$ whose reply is a single value $v\in\mathcal{V}_a$. For a candidate set $C$ define the branches
$C_v=\{x\in C \mid a(x)=v\}$,
and ignore values with $|C_v|=0$. If there is only one nonempty branch, then $a$ does not split $C$ and is not considered. Otherwise,
\begin{equation}
\mathrm{Cost}(C)=
\min_{\substack{a\in\mathcal{A}\\ |\{v\in\mathcal{V}_a \mid |C_v|>0\}|\ge 2}}
\left(
1+\sum_{v\in\mathcal{V}_a:\,|C_v|>0}\frac{|C_v|}{|C|}\,\mathrm{Cost}(C_v)
\right),
\label{eq:oqa-dp-kary}
\end{equation}
with the same terminal condition as above.

The recursion is evaluated over candidate sets. In practice, repeated subproblems can be avoided by caching $\mathrm{Cost}(C)$ keyed by a canonical representation of $C$, such as a sorted tuple of item identifiers. A fixed attribute order yields a reproducible oracle policy when several attributes tie.

To generate the oracle entropy traces used in plots, the optimal policy is rolled out for each fixed target $x^\star$. Starting from $C_0=\mathcal{X}$, at step $t$ the oracle selects an optimal query for $C_t$, reads the target's true reply, filters to $C_{t+1}$, and records $H_t=\log_2|C_t|$, with $t=0$ corresponding to the prior. The trajectory ends when the set is terminal, so tiers with duplicates can plateau at $H_t>0$.

\section{Scaling Up OQA Arbitrarily With Synthetic Attribute Tables}
\label{app:oqa-synthetic-scaling}

The real object tables used in OQA are small and fixed. For larger studies,
OQA also supports synthetic attribute tables that can be regenerated from a random seed at arbitrary
scale. A synthetic tier is specified by an attribute schema, a candidate count $N$, and a seed. The
generator writes the same artifacts as the released tiers. It outputs a JSON table of item attributes,
an \texttt{attributes.txt} file for binary tiers or a schema description for categorical tiers, an
\texttt{items.txt} list of identifiers, and an \texttt{equivalence\_classes.json} file that groups items
with identical attribute vectors when duplicates are present.

This setup serves two purposes. It enables scale, since $N$ can increase without changing the
evaluation code or the transcript rules. It also reduces semantic assistance. Item identifiers and
attribute names can be arbitrary strings, such as hexadecimal keys and short tokens, so performance
reflects query choice and consistent set updates rather than world knowledge.

In the binary synthetic tiers, fix $N$ and choose a Boolean dimension $d$ such that $2^d \ge N$.
This guarantees that at least $N$ distinct Boolean vectors exist. A simple construction samples $N$
distinct vectors in $\{0,1\}^d$, assigns each vector to a fresh identifier, and records the resulting
table. When $d$ is close to $\lceil \log_2 N \rceil$, the instance lies near the information limit,
since an ideal policy needs on the order of $\log_2 N$ questions. Taking $d$ larger adds redundancy,
which tests whether an agent can avoid attributes that do not help separate the remaining candidates.
Controlled ambiguity can be introduced by allowing duplicates, so the stopping rule ends on a
nontrivial equivalence class.

In the multiway synthetic tiers, choose a finite attribute set $\mathcal{A}$ with value sets
$\{\mathcal{V}_a\}_{a\in\mathcal{A}}$ such that $\prod_{a\in\mathcal{A}} |\mathcal{V}_a| \ge N$.
This ensures that at least $N$ distinct categorical vectors exist. Sample $N$ distinct vectors from
the Cartesian product and map each to a fresh identifier. Allowing duplicates creates plateaus where
several items share the same attribute vector. Difficulty is controlled by the branching factors
$|\mathcal{V}_a|$, by how balanced the value frequencies are, and by correlations among attributes.

For semantic domains such as Places, Cars, and Animals, listing the allowed attribute names is often
enough, since background knowledge can guide reasonable query orders. For fully synthetic domains
the names carry no meaning, so a planning test typically requires providing the table, or an
equivalent compact encoding, in the prompt. Two convenient encodings are \texttt{id:bitstring} for
binary tiers and \texttt{id:v1,v2,...} for $k$-ary tiers. Both grow linearly with $N$ and remain easy
to parse in long context settings.

The transcript-based metrics, namely $C_t$ and $H_t=\log_2|C_t|$, remain exact for any $N$ because
they rely only on deterministic filtering. Exact dynamic programming does not scale in the same way.
Its states are candidate subsets, and the number of reachable subsets can grow very quickly as $N$
increases. For larger tiers, greedy information gain, shallow lookahead, or sampling-based planners
provide strong approximate baselines, while the benchmark definition remains unchanged.

\begin{oqacode}{blue}{Deterministic synthetic OQA generator (binary and $k$-ary, reference implementation sketch)}
import json
import math
import random
from collections import defaultdict

def write_list(path, xs):
    with open(path, "w") as f:
        for x in xs:
            f.write(str(x) + "\n")

def equivalence_classes(table, attrs):
    # Group items that share the same attribute vector (binary or categorical).
    groups = defaultdict(list)
    for item_id, row in table.items():
        key = tuple(row[a] for a in attrs)
        groups[key].append(item_id)

    # Store only non-singletons (optional; keep all if preferred).
    out = {}
    for i, ids in enumerate(groups.values()):
        if len(ids) >= 2:
            out[f"class_{i:04d}"] = ids
    return out

def make_binary_oqa(
    N, d=None, seed=0, allow_duplicates=False, id_width=6
):
    rng = random.Random(seed)
    if d is None:
        # Slight margin reduces collisions when sampling without replacement.
        d = math.ceil(math.log2(max(N, 2))) + 2
    attrs = [f"a{i}" for i in range(d)]

    seen = set()
    table = {}
    while len(table) < N:
        bits = tuple(rng.getrandbits(1) for _ in range(d))
        if (not allow_duplicates) and (bits in seen):
            continue
        seen.add(bits)

        item_id = f"{len(table):0{id_width}x}"
        table[item_id] = {a: bool(b) for a, b in zip(attrs, bits)}

    return table, attrs

def make_kary_oqa(
    N, schema, seed=0, allow_duplicates=False, id_width=6
):
    # schema: dict attr -> list of allowed values (strings).
    rng = random.Random(seed)
    attrs = list(schema.keys())

    seen = set()
    table = {}
    while len(table) < N:
        vec = tuple(rng.choice(schema[a]) for a in attrs)
        if (not allow_duplicates) and (vec in seen):
            continue
        seen.add(vec)

        item_id = f"{len(table):0{id_width}x}"
        table[item_id] = {a: v for a, v in zip(attrs, vec)}

    return table, attrs

# Example usage:
# Binary: N=5000, d chosen automatically.
# table, attrs = make_binary_oqa(N=5000, seed=7, allow_duplicates=False)
# json.dump(table, open("data/synth_binary.json", "w"), indent=2)
# write_list("data/attributes.txt", attrs)
# write_list("data/items.txt", table.keys())
# json.dump(
#     equivalence_classes(table, attrs),
#     open("data/equivalence_classes.json", "w"),
#     indent=2,
# )

# Multiway: choose branching factors through schema.
# schema = {
#     "a0": ["v0", "v1", "v2", "v3", "v4"],
#     "a1": ["v0", "v1", "v2", "v3", "v4"],
#     "a2": ["v0", "v1", "v2", "v3"],
#     "a3": ["v0", "v1", "v2", "v3"],
# }
# table, attrs = make_kary_oqa(
#     N=20000, schema=schema, seed=11, allow_duplicates=False
# )
\end{oqacode}


\section{Prompt Autocompletion Experiment: Full Specification}
\label{app:prompt-autocomplete-spec}

Each run in the prompt autocompletion experiment has an unknown style variable, a short clarification phase with fixed question text, a single model call that produces the final description, and a deterministic verifier that scores the result. The policy controls only which style attribute to query next and when to stop.

Each task provides a product name and three feature phrases. The latent user intent is
\[
U=(\text{tone},\text{length},\text{format}),
\]
where tone, length, and format each have two listed values. The prior over $U$ is uniform and factorized. A policy may ask up to $K_{\max}\le 3$ clarification questions. Each clarification reveals the true value of exactly one component of $U$ through a simulated answer.

During belief updates, the implementation models the answer channel with symmetric label noise $\varepsilon=0.12$. After the policy stops, a single style estimate $\hat U$ is chosen by a MAP rule under the current belief. Ties are broken by \texttt{argmax} returning the first listed label in each attribute's ordering: \texttt{formal}, \texttt{short}, \texttt{bullets}. The final description is then generated with exactly one model call conditioned on $\hat U$. The run is scored only by the verifier described below.

The text of each clarification question is fixed, and only the choice of which attribute to ask is controlled by the policy:
\begin{itemize}\itemsep2pt
\item \texttt{tone}\quad \texttt{Before I write, do you prefer a formal or friendly tone?}
\item \texttt{length}\quad \texttt{Do you want a short or a medium length description?}
\item \texttt{format}\quad \texttt{Should I present it as bullet points or as a paragraph?}
\end{itemize}

Clarification turns are charged a fixed synthetic schedule that approximates the growth of context,
\[
\texttt{clarify\_cost}(k)\in\{24,48,72\}
\quad\text{for the $k$th clarification, with $k$ starting at 1.}
\]
Total tokens per run are computed as
\[
\texttt{tokens\_total}
=
\sum_{k=1}^{K}\texttt{clarify\_cost}(k)
+\texttt{tokens\_api},
\]
where \texttt{tokens\_api} is the model reported total token usage for the single generation call.

The final call sends a single text prompt to the Responses API, using the model \texttt{gpt-4o}. The prompt is formatted with two labeled sections, \texttt{[System]} and \texttt{[User]}, as shown below. The \texttt{[System]} section is constant. The \texttt{[User]} section is formed by inserting the task fields and the MAP style $\hat U$ into the template below.
\begin{lstlisting}[style=oqalist]
[System]
You are a concise marketing writer.
Follow tone, format, and length exactly.
Reuse the product features as is.

[User]
Product: <TASK.NAME>
Features: <TASK.FEATURE_1>, <TASK.FEATURE_2>, <TASK.FEATURE_3>
Tone: <formal|friendly>
Format: <bullets|paragraph>
Length: <short|medium>
Write the description now.
\end{lstlisting}

Given the generated text \texttt{out}, the task feature strings, and the latent $U$, the output is accepted if and only if all three conditions hold. First, every feature phrase appears literally as a substring of \texttt{out} after both are lowercased, so the check is \texttt{f.lower() in out.lower()} for each feature string \texttt{f}. Second, the format matches $U.\texttt{format}$. If $U.\texttt{format}=\texttt{bullets}$, \texttt{out} must contain at least two non empty lines whose first non whitespace character is \texttt{-}, \texttt{*}, or \texttt{\textbackslash item}. If $U.\texttt{format}=\texttt{paragraph}$, the output is rejected when it contains obvious bullet starts, which is tested by checking whether \texttt{out} contains \texttt{\textbackslash n-}, \texttt{\textbackslash n*}, or \texttt{\textbackslash n\textbackslash item}. Third, the length matches $U.\texttt{length}$. Let \texttt{wc} be the number of tokens in \texttt{out} when splitting on whitespace. If $U.\texttt{length}=\texttt{short}$, require $\texttt{wc}\le 70$. If $U.\texttt{length}=\texttt{medium}$, require $70\le \texttt{wc}\le 160$.

One run can be summarized compactly as follows. The product is Stainless Steel Bottle, with features insulated, 24 oz, and BPA free. The latent style is $U=\{\texttt{friendly},\texttt{short},\texttt{paragraph}\}$. A policy that queries all three attributes receives those three simulated answers, so $\hat U=U$. The model is then called once with the template above instantiated by these fields. The verifier returns \texttt{PASS} if the output contains all three feature substrings, contains no bullet lines, and has word count at most $70$.

\section{Prompt Optimization Experiment: Full Specification}
\label{app:promptopt-protocol}

Prompt optimization is treated as a small Bayesian bandit problem over a fixed set of prompt variants on ARC-Challenge multiple-choice questions. Each training step selects one prompt variant, makes exactly one model call at \texttt{temperature=0}, parses a single answer letter, and logs both correctness and the total token count returned by the API. Learning occurs only in the bandit state. The language model itself is not updated. The base model used in all calls is \texttt{gpt-4o}.

We load ARC-Challenge from the Hugging Face AI2 ARC dataset with the \texttt{ARC-Challenge} configuration. We filter both the training and validation splits to keep only examples with exactly four answer choices labeled \texttt{A}, \texttt{B}, \texttt{C}, and \texttt{D}. We reorder the choice texts into that order and drop any example whose \texttt{answerKey} is not one of those four labels.

Fix the set of prompt variants
\[
\begin{aligned}
P=\{&\texttt{letter\_only},\texttt{short\_reasoning},\\
&\texttt{eliminate\_two},\texttt{keyword\_match},\\
&\texttt{units\_and\_scales},\texttt{contrastive\_explanations}\}.
\end{aligned}
\]
Each $p\in P$ corresponds to a fixed system prompt. The user message has a fixed template shared across all variants. At step $t$, the policy chooses $p_t\in P$, formats one ARC example into the user template, prepends the system prompt for $p_t$, and calls the base model once. The API-reported total token usage for that call is denoted by $\tau_t$.

\begin{oqacode}{blue}{User message template shared by all prompts}
Question:
<QUESTION_TEXT>

Options:
A) <OPTION_A>
B) <OPTION_B>
C) <OPTION_C>
D) <OPTION_D>

Answer with only the single letter A, B, C, or D.
\end{oqacode}

\begin{oqacode}{teal}{System prompts for the six variants}
letter_only
You answer multiple choice science exam questions.
Read the question and options carefully and choose the single best answer.
Reply with only the option letter (A, B, C, or D). Do not include any text besides the letter.

short_reasoning
You are an expert science tutor taking a multiple choice test.
Privately reason step by step, but reply with only the single best option letter.
Never include the reasoning in your reply.

eliminate_two
You answer multiple choice science questions by eliminating clearly wrong options first.
Mentally eliminate at least two options and then pick the best remaining one.
Reply with only the option letter (A, B, C, or D).

keyword_match
You are a careful science exam solver.
Match key concepts and entities in the question to the options, ignore distractors, and pick the best match.
Reply with only the option letter (A, B, C, or D).

units_and_scales
You are a science exam assistant.
Pay special attention to units, scales, and quantitative relationships when choosing among the options.
Reply with only the option letter (A, B, C, or D).

contrastive_explanations
You answer science questions by comparing how each option would explain the question.
Mentally contrast the options and pick the one that best explains the described situation.
Reply with only the option letter (A, B, C, or D).
\end{oqacode}

A deterministic checker maps the raw model output at step $t$ to a binary outcome. It extracts the first standalone letter in $\{\texttt{A},\texttt{B},\texttt{C},\texttt{D}\}$ using a case-insensitive word-boundary match, for example the regex \texttt{\textbackslash b([ABCD])\textbackslash b}. If no such letter is found, the step is marked incorrect. Otherwise the step outcome is $y_t=1$ if the parsed letter equals the gold \texttt{answerKey}, and $y_t=0$ otherwise.

All policies train on the same fixed subset of at most $400$ ARC-Challenge training questions, sampled without replacement using a fixed random seed. Training stops when this subset is exhausted or when cumulative training tokens reach or exceed a budget $B$, with the reference setting $B=120{,}000$. (Because the budget check happens after each model call, the final total can exceed $B$ by at most one call.) Evaluation uses a shared holdout subset of $200$ ARC-Challenge validation questions, also sampled once with a fixed seed. Policy comparisons therefore differ only through the sequence of prompt variants selected during training.

Each prompt $p\in P$ is treated as an arm with an unknown Bernoulli accuracy $\theta_p$. Independent Beta posteriors are maintained,
\[
q_t(\theta_p)=\mathrm{Beta}(\alpha_p,\beta_p),
\qquad
\alpha_p=\beta_p=1 \ \text{at initialization}.
\]
After testing $p_t$ and observing $y_t\in\{0,1\}$, the conjugate update is
\[
\alpha_{p_t}\leftarrow \alpha_{p_t}+y_t,
\qquad
\beta_{p_t}\leftarrow \beta_{p_t}+(1-y_t),
\]
and the posterior mean accuracy is
\[
m_p \equiv \mathbb{E}[\theta_p]=\frac{\alpha_p}{\alpha_p+\beta_p}.
\]

Token cost affects only acquisition. Each arm maintains an exponential moving average estimate $\hat c_p$ of its token cost, initialized at $\hat c_p=120$ for all prompts. After an evaluation of $p_t$ with observed cost $\tau_t$, the update is
\[
\hat c_{p_t}\leftarrow \gamma\,\hat c_{p_t} + (1-\gamma)\,\tau_t,
\qquad
\gamma=0.9.
\]

Let $P^\star=\arg\max_{p\in P}\theta_p$ denote the identity of the best prompt under the latent accuracies. Its posterior is approximated by Monte Carlo sampling. Draw $\tilde\theta_p^{(s)}\sim \mathrm{Beta}(\alpha_p,\beta_p)$ independently for $s=1,\dots,S$ and all $p\in P$, with the reference choice $S=600$ during training. Compute $\tilde P^{\star(s)}=\arg\max_p \tilde\theta_p^{(s)}$ and form the empirical distribution
\[
\hat \pi_t(p)=\frac{1}{S}\sum_{s=1}^{S}\mathbf{1}\{\tilde P^{\star(s)}=p\}.
\]
Uncertainty is summarized by the entropy in bits
\[
H_t \equiv H(P^\star)\approx -\sum_{p\in P}\hat \pi_t(p)\log_2 \hat \pi_t(p).
\]
For plotting, the information efficiency of step $t$ is
\[
\mathrm{IE}_{1\mathrm{k}}(t)=1000\cdot \frac{\Delta H_t}{\max\{\tau_t,1\}},
\qquad
\Delta H_t=\max\{H_t-H_{t+1},0\},
\]
where $H_t$ is computed before the Beta update and $H_{t+1}$ is computed after the update.

At each step, the policy chooses one prompt $p_t$, performs one model call, then updates $(\alpha,\beta)$ and $\hat c$. Under the mutual information rule, the expected next entropy after one more labeled sample on $p$ is approximated by
\[
\mathbb{E}[H_{t+1}\mid p]
\approx
m_p\,H(P^\star\mid \alpha_p{+}1,\beta_p)
+
(1-m_p)\,H(P^\star\mid \alpha_p,\beta_p{+}1),
\]
where each entropy term is recomputed by the same Monte Carlo procedure after the hypothetical update on $p$ only. The acquisition score divides expected entropy reduction by estimated token cost,
\[
u_t^{\mathrm{MI}}(p)=\frac{\max\{H_t-\mathbb{E}[H_{t+1}\mid p],0\}}{\max\{\hat c_p,1\}},
\qquad
p_t=\arg\max_{p\in P} u_t^{\mathrm{MI}}(p).
\]

Under the knowledge gradient rule, let $m_{\max}=\max_{j\in P} m_j$ be the current best posterior mean. If prompt $p$ were sampled and succeeded, its mean would become $(\alpha_p{+}1)/(\alpha_p{+}\beta_p{+}1)$. If it failed, its mean would become $\alpha_p/(\alpha_p{+}\beta_p{+}1)$. Let $m_{\max}^{\mathrm{succ}}(p)$ and $m_{\max}^{\mathrm{fail}}(p)$ be the resulting best means across prompts under these two hypothetical updates. Define
\begin{align}
\mathrm{KG}_t(p)
&=
m_p\,\max\{m_{\max}^{\mathrm{succ}}(p)-m_{\max},0\}
+(1-m_p)\,\max\{m_{\max}^{\mathrm{fail}}(p)-m_{\max},0\},\nonumber\\
u_t^{\mathrm{KG}}(p)
&=\frac{\mathrm{KG}_t(p)}{\max\{\hat c_p,1\}},
\end{align}
and select $p_t=\arg\max_{p\in P} u_t^{\mathrm{KG}}(p)$.

A mixed rule combines the epistemic score $u_t^{\mathrm{MI}}$ and the pragmatic score $u_t^{\mathrm{KG}}$,
\[
s_t(p)=\lambda_t\,u_t^{\mathrm{MI}}(p) + (1-\lambda_t)\,u_t^{\mathrm{KG}}(p),
\qquad
p_t=\arg\max_{p\in P} s_t(p).
\]
In the notebook implementation, this mixed rule is computed as an Expected Free Energy score
$G_t(p)=-s_t(p)$ and the selected prompt is $p_t=\arg\min_p G_t(p)$, which is equivalent.
The weight decays linearly over the available training steps,
\[
\lambda_t=\max\!\left(0.1,\;1-\frac{t}{T-1}\right),
\]
where $T$ is the number of available training steps, at most $400$. Thompson sampling draws one $\tilde\theta_p\sim\mathrm{Beta}(\alpha_p,\beta_p)$ per prompt and chooses $p_t=\arg\max_{p\in P}\tilde\theta_p$. Round robin cycles deterministically through the prompt set with $p_t$ equal to the $t \bmod |P|$ element under a fixed ordering of $P$.

After training, the selected prompt is the posterior mean maximizer $\hat p=\arg\max_{p\in P} m_p$. This single prompt is then evaluated on the shared holdout set using the same user template and the same parsing rule. Reported holdout accuracy is the fraction of correct parses, and token statistics are computed from the API token usage for those calls.

As a concrete example, suppose $p_t=\texttt{units\_and\_scales}$, the parsed model reply is \texttt{B}, the gold key is \texttt{B}, and the API reports $\tau_t=134$. Then $y_t=1$, the posterior for that prompt updates from $\mathrm{Beta}(1,1)$ to $\mathrm{Beta}(2,1)$, and its cost estimate updates to $\hat c\leftarrow 0.9\,\hat c + 0.1\cdot 134$. All acquisition scores for the next step are computed from the updated posteriors and cost estimates.


\section{Jailbreaking and Defensive Design}
\label{sec:jailbreak_bilevel}

Adaptive prompt attacks can be viewed as a learning process. Each exchange provides feedback, and that feedback can be used to refine later attempts. The same information accounting used in benign interaction can therefore be turned around and used to describe what an attacker learns about a defense surface. The aim here is to formalize this structure and to suggest defensive objectives. It is not to propose new attack methods.

Let $\tilde x_t$ denote the visible text presented to the model at turn $t$. Let
\[
d=(s,\phi,\kappa,\pi_{\mathrm{clar}},g,f,\rho)
\]
collect the system prompt $s$, guard parameters $\phi$, decoding constraints $\kappa$, a clarifying policy $\pi_{\mathrm{clar}}$, tool and data access rules $g$, filters $f$, and rate limits $\rho$. Let $d_{\mathrm{pub}}$ denote the components of $d$ observable to the attacker. An action $a_t$ represents any intervention that changes what the model sees next, such as editing the prompt or influencing retrieved context. Given $(\tilde x_t,a_t,d)$, the system produces a response $Y_t$ with
\[
Y_t \sim p_\theta(y\mid \tilde x_t,a_t,d).
\]

Let $\mathcal{U}$ denote a set of undesired outcomes and let $V_t=\mathbf{1}\{Y_t\in\mathcal{U}\}$ be the associated indicator at turn $t$. In tool using systems, $\mathcal{U}$ can be understood broadly. It can include unsafe text, unauthorized tool actions, or leakage of sensitive data through outputs or tool calls. Let $V=\mathbf{1}\{\exists t\le T:\,V_t=1\}$ denote the session level failure indicator. Let $Z_t$ denote an observable signal derived from the $t$-th exchange, such as a refusal flag, the returned text, tool call traces, error messages, or coarse timing and length cues.

A stylized adaptive attacker can be modeled with latent variables $\Xi$ that represent unknown aspects of the system that are useful for exploitation. Examples include which instruction sources dominate, which filters fail under stress, or which tool routes expose sensitive resources. Let $h_{t-1}$ denote the interaction history available to the attacker up to turn $t-1$. Starting from a prior $q_1(\Xi)$, after choosing $a_t$ and observing $Z_t$, the attacker updates a belief by Bayes
\[
q_{t+1}(\Xi)\propto q_t(\Xi)\,p(Z_t\mid \Xi,h_{t-1},a_t,d_{\mathrm{pub}}).
\]
A single step action rule that makes the exploration and exploitation trade explicit is
\begin{equation}
a_{t}\in\arg\max_{a\in\mathcal{A}}
\Big[
\mathbb{E}[V_t\mid h_{t-1},a,q_t,d_{\mathrm{pub}}]
+\lambda\,I_{q_t}(\Xi;Z_t\mid h_{t-1},a,d_{\mathrm{pub}})
\Big]
\quad\text{s.t.}\quad
\sum_{k=1}^{t}\mathrm{tokens}(a_k)\le B.
\label{eq:att_outer_short_app}
\end{equation}
The first term rewards actions that are likely to succeed immediately. The second rewards actions whose observable consequences reduce uncertainty about $\Xi$, which can make later actions more effective.

A defender chooses $d$ once and then faces an adaptive policy. Let $\alpha$ denote an attacker policy that maps the interaction history to actions and obeys the same budget. Let $\mathsf{intent}\in\{\mathrm{benign},\mathrm{adversarial}\}$ denote the session type. A robust design objective can be written as
\begin{equation}
\begin{aligned}
\min_{d\in\mathcal{D}}
\;&
\sup_{\alpha\in\Pi_B}
\Big[
\mathbb{E}[V\mid d,\alpha]
+\lambda_{1} I_{q_1}(\Xi;Z_{1:T}\mid d_{\mathrm{pub}},\alpha)
+\lambda_{2}\mathrm{cost}(d)
\Big]\\
\text{s.t.}\;&
\mathbb{E}[\mathrm{utility}\mid d,\ \mathsf{intent}=\mathrm{benign}]\ge \tau.
\end{aligned}
\label{eq:defense_minmax_short_app}
\end{equation}
where $Z_{1:T}$ denotes the sequence of observed signals, $\Pi_B$ is the set of budget constrained adaptive attacker policies, and the mutual information term is computed under the attacker prior $q_1$. The information term penalizes designs that allow rapid learning about $\Xi$ through interaction. The constraint enforces acceptable utility on benign sessions.

It is often useful to report a token normalized leakage rate. Define
\[
\mathrm{IE}_{1\mathrm{k}}
=
1000\,
\frac{\sum_{t=1}^{T}\Delta I_t}{\sum_{t=1}^{T}\mathrm{tokens}(a_t)},
\qquad
\Delta I_t
=
\mathrm{KL}\!\big(q_{t+1}(\Xi)\,\|\,q_t(\Xi)\big).
\]
Under a Bayesian update, $\Delta I_t$ is the realized information gain about $\Xi$ from observing $Z_t$. A bound of the form $\mathrm{IE}_{1\mathrm{k}}\le \eta$ limits the total information learned about $\Xi$ to scale linearly with the total interaction budget.

Several practical controls map cleanly onto these terms. Refusal shaping aims to make the refusal branch as uninformative as possible. Let $F_t$ denote the refusal indicator at turn $t$. In an idealized model with no other side channels, drawing refusals from a fixed distribution that does not depend on $\Xi$ yields $I(\Xi;Y_t\mid F_t=1,d_{\mathrm{pub}})=0$. In practice, it also helps to control easy cues such as response length, formatting, and latency. Separation of authority targets indirect prompt injection and related failures where untrusted text is treated as control. The configuration $d$ can restrict tool permissions, keep secrets out of the model context, and ensure that untrusted content cannot directly trigger privileged actions. Deterministic validation of outputs and tool calls, and explicit approval for high risk actions, reduce the value of any single successful manipulation. Finally, budgets and rate limits directly bound the total interaction length (with output length capped by $\kappa$) and therefore bound total leakage even when per step leakage cannot be driven to zero.

Information terms are difficult to estimate in real systems, and $\Xi$ is a modeling choice rather than a directly observable quantity. In practice, mutual information is approximated using surrogate models, calibrated uncertainty estimates, and Monte Carlo evaluation. The main point is structural, i.e., defense can be framed not only as reducing the chance of an undesired outcome, but also as reducing the rate at which an attacker learns from interaction.

\section{The Active Inference View of Bilevel Reinforcement Learning}
\label{sec:bilevel_rl_vs_ai}

Bilevel reinforcement learning and active inference share a two level structure. An outer choice fixes the conditions of a control problem, and an inner solution responds by producing behavior that is optimal under those conditions. In reinforcement learning the inner object is a policy that maximizes expected return in an induced MDP. In active inference the inner object is an inference step that updates beliefs, and the outer object selects actions that balance preferred outcomes against uncertainty reduction.

In bilevel reinforcement learning, a designer selects configuration variables $x$, such as reward weights, prompt parameters, or safety settings. A context $\xi$ is drawn from a distribution $\mathcal{D}$, and together $(\xi,x)$ define an MDP $M(\xi,x)$ with horizon $H$. The agent computes a best response policy
\[
\pi^{*}(\cdot\mid \xi,x)\in\arg\max_{\pi} J(\pi;\xi,x),
\qquad
J(\pi;\xi,x)=\mathbb{E}\!\Big[\sum_{t=1}^{H} r_t\Big].
\]
The outer problem chooses $x$ to optimize expected performance under this response, often with a regularizer $\Omega$,
\[
\min_{x}\ -\mathbb{E}_{\xi\sim\mathcal{D}}\!\big[J(\pi^{*};\xi,x)\big] + \Omega(x).
\]
This captures contextual Stackelberg style design problems, where a leader selects $x$ and a follower solves the induced MDP, and it also covers outer loop design in modern learning systems \citep{thoma2024cbrl_neurips,shen2024pmlr}.

Active inference evaluates behavior through expected free energy. Preferences are encoded by a distribution $p^{\star}(o_t)$ over observations, and latent variables $\Theta$ represent what remains uncertain and worth inferring. A horizon objective that matches the standard risk and epistemic decomposition can be written as
\[
\mathcal{G}_{H}(\pi;x)
=
\mathbb{E}\!\left[\sum_{t=1}^{H}
-\ln p^{\star}(o_t)
+
\mathcal{A}_t
-
\lambda\,I(\Theta; o_t \mid h_{t-1},\pi,x)
\right],
\]
where $h_{t-1}$ denotes the history up to time $t-1$, and $\mathcal{A}_t$ is an ambiguity term that penalizes observations that remain noisy even when $\Theta$ is known. All expectations, entropies, and mutual informations are taken under the predictive distribution induced by the generative model and the policy $\pi$, conditional on $h_{t-1}$ and $x$. A convenient choice is
\[
\mathcal{A}_t
=
\mathbb{E}\!\left[H\!\left(o_t \mid \Theta, h_{t-1}, \pi, x\right)\right].
\]
When the observation model is deterministic given $\Theta$ and the chosen action, $\mathcal{A}_t$ vanishes and the epistemic term reduces to expected information gain. In that deterministic case, minimizing expected free energy reduces to minimizing expected risk while maximizing information gain \citep{friston2017process,parr2019generalised}.

The link to reward maximization becomes explicit once preferences are written as reward. Define
\[
r_t^{\mathrm{pref}}=\ln p^{\star}(o_t).
\]
Maximizing $\sum_t r_t^{\mathrm{pref}}$ is equivalent to minimizing the cumulative risk term $\sum_t -\ln p^{\star}(o_t)$. If an intrinsic term is added that rewards information gain, then a return objective aligns with the expected free energy form. One simple epistemic reward is
\[
r_t^{\mathrm{epi}}
=
\lambda\,I(\Theta; o_t \mid h_{t-1},\pi,x),
\]
together with a penalty that matches ambiguity. The combined return
\[
\tilde r_t
=
r_t^{\mathrm{pref}}
+
r_t^{\mathrm{epi}}
-
\mathcal{A}_t
\]
mirrors the expected free energy objective up to constants and sign conventions. The alignment is at the level of objectives. The modeling stance still differs, since active inference represents goals as preferences over observations and treats belief dynamics as part of the state, while reinforcement learning often treats rewards as primitive and leaves inference implicit.

When information acquisition has an operational cost, the outer design problem can trade off task value, cost, and information. It is often useful to separate information that supports performance from information that should be difficult to extract. Let $c_t(x)$ denote operational cost, such as token usage or tool latency. Let $\Theta_{\mathrm{task}}$ denote variables whose rapid identification is beneficial, and let $\Theta_{\mathrm{sens}}$ denote variables whose rapid identification is undesirable. Define
\[
\mathcal{I}_{\mathrm{task}}(x)=\mathbb{E}\!\Big[\sum_{t=1}^{H} I(\Theta_{\mathrm{task}};o_t\mid h_{t-1},\pi^*,x)\Big],
\quad
\mathcal{I}_{\mathrm{sens}}(x)=\mathbb{E}\!\Big[\sum_{t=1}^{H} I(\Theta_{\mathrm{sens}};o_t\mid h_{t-1},\pi^*,x)\Big].
\]
One possible outer objective is
\[
\min_{x}\ 
-\mathbb{E}_{\xi}\!\big[J(\pi^{*};\xi,x)\big]
+
\mu\,\mathbb{E}\!\Big[\sum_{t=1}^{H} c_t(x)\Big]
-
\gamma\,\mathcal{I}_{\mathrm{task}}(x)
+
\eta\,\mathcal{I}_{\mathrm{sens}}(x),
\]
with $\gamma,\eta,\mu\ge 0$. The same bilevel machinery applies, since $\pi^{*}$ is still the inner best response, and the outer variables can be updated using hypergradients that differentiate through this inner solution, as in BLPO \citep{prakash2025blpo}.

For language model agents, clarifying questions and tool calls are actions that purchase information at a measurable token cost. The preceding objectives make the trade between utility, cost, and uncertainty reduction explicit. Increasing the weight on task relevant information gain encourages earlier probing, and increasing the weight on sensitive information discourages probing that may overreach, while keeping the inner and outer roles of inference and control distinct.


\end{document}